\PassOptionsToPackage{dvipsnames,table}{xcolor}
\documentclass[12pt]{article}
\usepackage[margin=1in]{geometry}

\usepackage[round]{natbib}

\usepackage[utf8]{inputenc}
\usepackage[T1]{fontenc}
\usepackage[english]{babel}
\usepackage{hyperref}
\usepackage{url}
\usepackage{booktabs}
\usepackage{amsfonts}
\usepackage{nicefrac}
\usepackage{microtype}
\usepackage[dvipsnames]{xcolor}
\definecolor{quotedliterature}{RGB}{255,245,255}
\hypersetup{hidelinks}

\usepackage{graphicx}
\usepackage{subcaption}
\usepackage{pifont}
\usepackage{mathtools}
\usepackage{amssymb}
\usepackage{amsthm}
\usepackage{mathrsfs}
\usepackage{bm}
\usepackage{enumitem}
\usepackage{algorithm}
\usepackage{algorithmic}
\usepackage{wrapfig}
\usepackage{float}
\usepackage[capitalize]{cleveref}

\graphicspath{{figures/}}
\newtheorem{lem}{Lemma}

\newtheorem{thm}{Theorem}

\theoremstyle{remark}

\newcommand{\R}{\mathbb{R}}

\newcommand{\cP}{\mathcal{P}}

\newcommand{\cN}{\mathcal{N}}

\newcommand{\cL}{\mathcal{L}}

\newcommand{\bes}{\begin{equation*}}
\newcommand{\ees}{\end{equation*}}
\newcommand{\beas}{\begin{eqnarray*}}
\newcommand{\eeas}{\end{eqnarray*}}
\newcommand{\bea}{\begin{eqnarray}}
\newcommand{\eea}{\end{eqnarray}}
\newcommand{\be}{\begin{equation}}
\newcommand{\ee}{\end{equation}}
\newcommand{\bei}{\begin{itemize}}
\newcommand{\eei}{\end{itemize}}
\newcommand{\bec}{\begin{cases}}
\newcommand{\eec}{\end{cases}}
\newcommand{\ben}{\begin{enumerate}}
\newcommand{\een}{\end{enumerate}}

\newcommand{\bbE}{\mathbb{E}}

\newcommand{\bbl}{\begin{block}}
\newcommand{\ebl}{\end{block}}

\newcommand{\dd}{\,\mathrm{d}}

\newcommand{\rme}{\mathrm{e}}

\newcommand*{\defeq}{\coloneqq}

\newcommand{\msf}[1]{\mathsf{#1}}

\newcommand{\eps}{\varepsilon}

 \newcommand{\bb}{\mathbb}

 \newcommand{\rset}{\mathbb{R}}

 \def\rset{\mathbb{R}}

\newcommand{\ccint}[1]{\left[#1\right]}

\newcommand{\Idd}{\operatorname{I}_d}
\newcommand{\gauss}{\mathbf{N}}
\newcommand{\rmd}{\mathrm{d}}

\usepackage{multirow}

\newcommand{\PE}{\mathbb{E}}
\newcommand{\1}{\bm{1}}

\definecolor{stdgray}{gray}{0.48}
\newcommand{\stddev}[1]{\mathbin{\textcolor{stdgray}{\pm}}\textcolor{stdgray}{#1}}
\providecommand{\wentry}[2]{#1\,{\color{stdgray}$\pm$\,#2}}

\theoremstyle{plain}
\newtheorem{theorem}{Theorem}[section]

\newtheorem{lemma}[theorem]{Lemma}

\theoremstyle{definition}

\theoremstyle{remark}
\newtheorem{remark}[theorem]{Remark}

\title{Trajectory inference via Acceleration  Matching}

\author{
\begin{tabular}{cc}
Bartolo Dazzini\thanks{Department of Mathematics, University of Padova. \texttt{bartolo.dazzini@studenti.unipd.it}}
&
Giovanni Conforti\thanks{Department of Mathematics, University of Padova. \texttt{giovanni.conforti@unipd.it}}
\\[0.6em]
Alain Durmus\thanks{CMAP, \'Ecole polytechnique. \texttt{alain.durmus@polytechnique.edu}}
&
Aram-Alexandre Pooladian%
\thanks{Institute for Foundations of Data Science, Yale University. \texttt{aram-alexandre.pooladian@yale.edu}}
\thanks{Corresponding author.}
\end{tabular}
}

\date{\today}

\begin{document}

\maketitle

\begin{abstract}
Trajectory inference is a fundamental problem in many scientific domains: given a collection of unpaired snapshots of observations at discrete time points, the goal is to generate smooth trajectories that best resemble and interpolate the data.
Existing numerical approaches to this problem have various computational burdens, such as explicitly computing spline constructions or simulating trajectories during training.
In order to overcome these limitations, we propose a new algorithm called Acceleration Matching (\texttt{AM}). Our approach consists of lifting the original interpolation problem to phase space and then regressing onto an explicit conditional acceleration field that induces random trajectories with smooth positional paths that, under suitable assumptions, agree with the prescribed marginals. Importantly, our resulting training algorithm only requires positional data, avoids trajectory simulation during training, and uses direct Gaussian conditioning to sample training states and targets. {We provide ample numerical evidence suggesting that \texttt{AM} is competitive with or superior to existing algorithms on several benchmark problems from the existing literature.}
\end{abstract}

\section{Introduction}
We investigate the trajectory inference problem: given snapshots of unpaired observations collected at multiple time points
$
0=t_0<t_1<\cdots<t_J=1$, the goal is to generate trajectories that fit the given marginal data. In probabilistic terms, we want to find a measure-valued curve $t \in \ccint{0,1} \mapsto \mu_t \in {\cal P}(\R^d)$  such that
\begin{align}\label{eq:marg_constraints}
     \mu_{t_j} = \rho_{t_j} \quad \forall j \in \{0,1,\ldots,J\}\,,
\end{align}
for the given (empirical) probability distributions $\{\rho_{t_j}\}_{j=0}^J$. This problem appears in a variety of data-driven research domains, including astronomy, economics, and single-cell RNA sequencing, among others \citep{franzke2015stochastic,kazakevivcius2021understanding,yeo2021generative,bunne2022proximal,kubica2007efficient,liounis2020autonomous}. In each of these cases, data are obtained over time in the form of unpaired observations. For example, the destructive nature of single-cell RNA sequencing makes it impossible to have repeated measurements of the same cell over time as the population (of cells) continues to evolve. As a result, the nature of the data renders standard procedures (based on, for example, time-series modeling) unhelpful. These unique challenges have led to several mathematical and algorithmic developments for the problem; see~\citet{lavenant2024toward, hong2025trajectory, bunne2023learning, schiebinger2019optimal}, among many others.

A naive approach is to simply construct $(\mu_t)_{t \in [0,1]}$ in a piecewise manner, learning $J$ independent segments of the probability path of interest. This can be accomplished by separately learning transport maps between each $\rho_{t_j}$ and $\rho_{t_{j+1}}$ and then stitching the paths together. For instance, one can learn optimal transport maps \citep{manole2021plugin,pooladian2021entropic,deb2021rates} or instead pass through the popular flow matching framework \citep{lipman2022flow,albergo2022building,liu2022flow}.
\begin{figure}[t]
    \centering
    \begin{subfigure}[t]{0.48\linewidth}
        \centering
        \includegraphics[width=\linewidth]{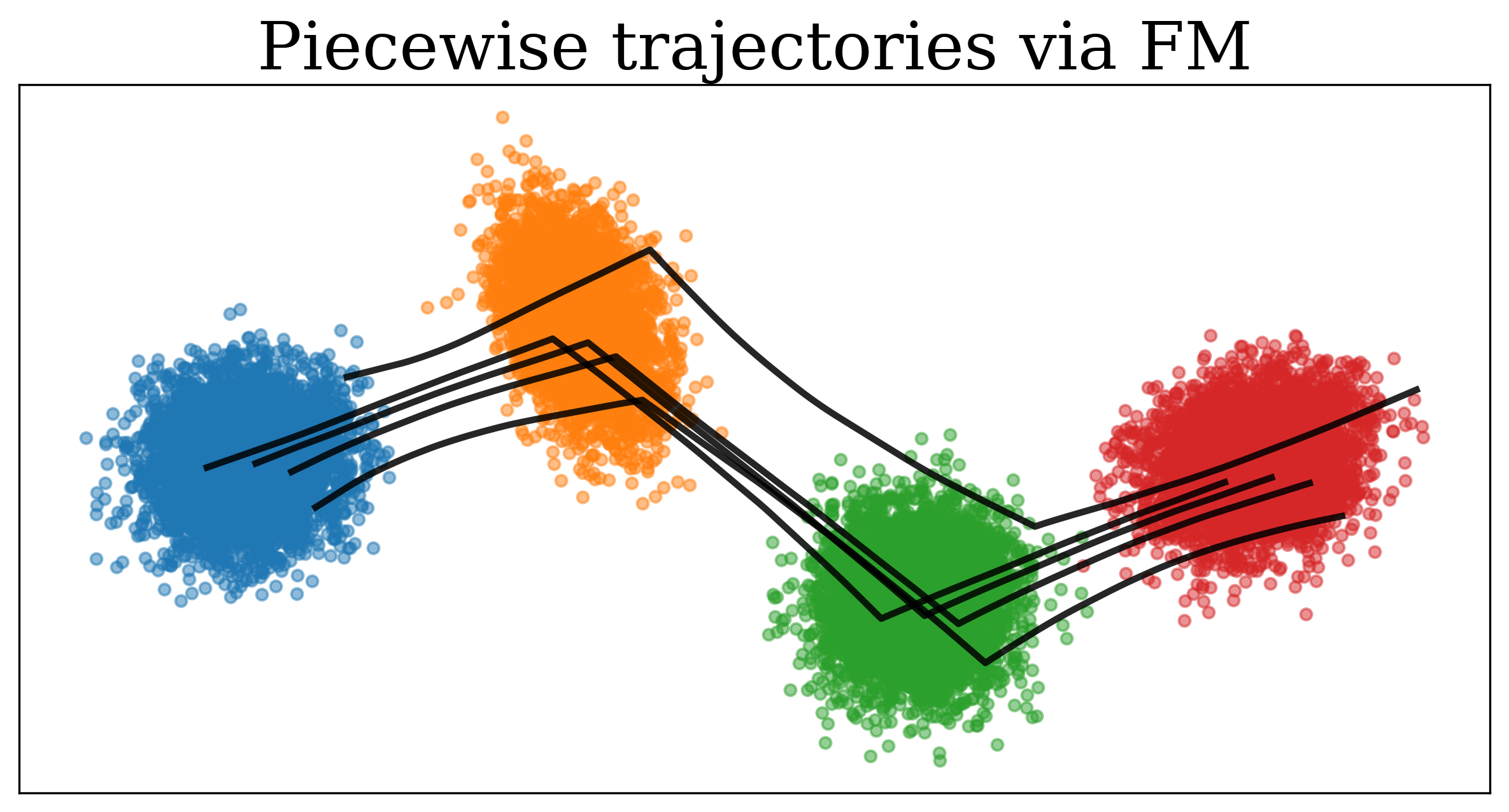}
    \end{subfigure}
    \hfill
    \begin{subfigure}[t]{0.48\linewidth}
        \centering
        \includegraphics[width=\linewidth]{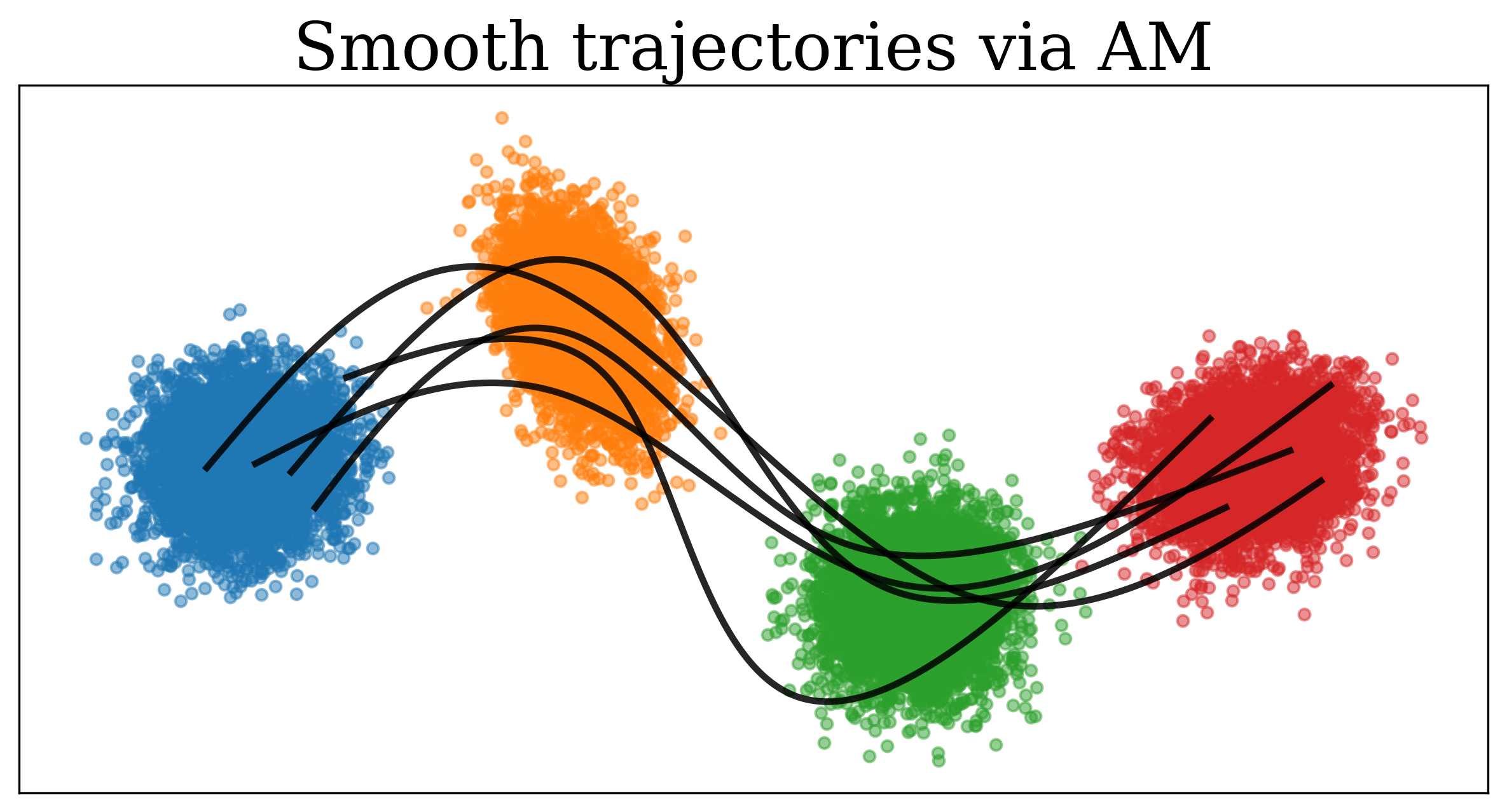}
    \end{subfigure}
    \caption{{\texttt{AM} outputs smoother trajectories \textbf{(right)} than piecewise-linear FM \textbf{(left)}.}}
    \label{fig:piecewise_v_splines}
\end{figure}

Depending on the method, the learn-then-stitch approach unfortunately requires, in the worst case, solving $J$ separate learning problems. Moreover, without a smooth interpolation construction, this naive approach can produce trajectories that lack global smoothness, leading to non-smooth, potentially unrealistic trajectories. 
The left side of Figure~\ref{fig:piecewise_v_splines} displays trajectories given by training a single time-dependent vector field using the flow matching (FM) framework, with linear interpolation targets on consecutive intervals; the trajectories are visibly non-smooth due to the abrupt discontinuity in the velocity field as it switches between the marginals. 

Our goal, by contrast, is to generate smooth positional trajectories by construction. To this end, we lift the original problem and consider trajectories given by stochastic differential equations (SDEs) on phase space (see Section~\ref{sec:phasespace} for more background), which are of the form
\begin{align}\label{sde_pp}
        \dd X_t = V_t \dd t\,, \quad \dd V_t = a_t(X_t, V_t) \dd t + \sqrt{\varepsilon} \dd B_t\,,
\end{align}
where $a_t : \R^{2d}\to\R^d$ is an acceleration field and $(B_t)_{t\in[0,1]}$ is standard Brownian motion with diffusion parameter $\eps > 0$. These dynamics generate smooth trajectories by design, as the velocities now have continuous sample paths and thus the positional paths are continuously differentiable. Our approach then consists of defining and targeting an appropriate acceleration field that yields dynamics which satisfy the marginal constraints ${\rm Law}(X_{t_j}) = \rho_{t_j}$ for all $j \in \{0,1,\ldots,J\}$.

While we are not the first to consider dynamics of the form \eqref{sde_pp}, related prior work {uses} trajectory simulation or iterative conditional-velocity propagation within their training pipelines that render their methods computationally intensive {\citep{chen2023deep,theodoropoulos2025momentum};} see Section~\ref{sec:related_work} for a thorough comparison.

\paragraph{Contributions.}
In this work, we present \emph{Acceleration Matching} (\texttt{AM}), a novel approach for trajectory inference. Our algorithm is based on the Markovianization of kinetic Brownian bridges in phase space (see Lemma~\ref{lem:bridge_process} and Theorem~\ref{thm:main} in Section~\ref{sec:condacc} for more precise details) and can thus be viewed as an extension of the  flow matching framework \citep{lipman2022flow,albergo2022building,liu2022flow} for learning second-order dynamics in the presence of multiple marginal constraints. Concretely, we jointly sample velocities conditioned on all position knots, then sample intermediate bridge states and explicit acceleration targets. 
This approach yields a training objective that only assumes access to positional data and requires no numerical trajectory rollouts, yet enjoys a notion of smoothness along the entire trajectory. Importantly, under the stated regularity assumptions, the population acceleration field satisfies the marginals constraints defined by \eqref{eq:marg_constraints} (see Theorem~\ref{thm:main}).
From the right-hand side of Figure~\ref{fig:piecewise_v_splines}, we see that trajectories generated by \texttt{AM} are visibly smoother than those generated by FM.
We demonstrate the practical viability of \texttt{AM} on several standard datasets from the literature in Section~\ref{sec:results}, where our approach achieves competitive or superior held-out marginal errors, with performance varying across datasets and couplings. 

\section{Background}\label{sec:prelim}
\subsection{Notation}

\label{sec:notation}
Let $\cP(\R^d)$ denote the set of probability measures over $\R^d$. For a collection of probability distributions $\{P_i\}_{i=1}^M \subset \cP(\R^d)$, let $\Pi(P_1,\ldots,P_M)$ denote the set of couplings between the measures i.e., for $\pi \in \Pi(P_1,\ldots,P_M)$, $\pi \in \cP((\R^{d})^M)$ such that for every $i \in \{1,\ldots,M\}$, the $i^{\rm th}$ marginal of $\pi$ is $P_i$. 
Let ${\cal C}^d \defeq C([0,1];\R^{d})$ be the space of continuous functions from the unit interval to $\R^{d}$ for $d \in \mathbb{N}$, endowed with the uniform topology and the corresponding Borel $\sigma$-field $\mathcal{B}({\cal C}^d)$. The set of probability measures over the resulting measurable space is denoted by $\mathcal{P}({\cal C}^d)$.
We use bold-font vectors (e.g., $\bm{t},\bm{x},\bm{v}$) to denote a concatenated vector; for example, $\bm{t} = (t_0,t_1,\ldots,t_J) \in \R^{J+1}$ and $\bm{x} = (x_0,x_1,\ldots,x_J) \in \R^{d(J+1)}$ (and similarly for $\bm v$).
For a set $A$, let $\bm{1}_A$ denote the $0$-$1$ indicator function over $A$.
For a smooth function $F:\R^{d}\times \R^d\to \R$ written $(x,v)\mapsto F(x,v)\in \R$, we write $\nabla_x F(x,v)$ and $\nabla_v F(x,v)$ to explicitly denote with respect to which argument the gradient is taken.

\subsection{Evolution of second-order probability flows}\label{sec:phasespace}
In this section, we briefly review probability paths on phase space, the joint space of positions $x\in\R^d$ and velocities $v \in \R^d$ of a given physical system. Throughout, we write phase space as $\rset^{2d}$.

Our approach is based on processes defined as solutions of second-order SDEs of the form 
\begin{align}\label{eq:fp_dyn}
    \dd X_t^a = V_t^a \dd t\,, \quad \dd V_t^a = a_t(X_t^a,V_t^a)\dd t + \sqrt{\eps} \dd B_t\,, \quad t \in\ccint{0,1}{\,,}
\end{align}
with initial conditions $(X_0^a,V_0^a)$, 
where $(B_t)_{t \in [0,1]}$ is {$d$-dimensional} standard Brownian motion, $\eps > 0$, and $(t,x,v) \in \ccint{0,1}  \times \rset^{2d} \mapsto a_t(x,v) \in \R^d$. We then naturally interpret $(a_t)_{t\in[0,1]}$ as an inhomogeneous acceleration field that governs the underlying dynamics.

As an example, if $a_t(x,v) = - \gamma v$ for a parameter $\gamma > 0$, the resulting process is called \emph{kinetic underdamped Brownian motion}, denoted by $\msf R^{\gamma,\eps} \in \mathcal{P}({\cal C}^{2d})$. If $\gamma = 0$ (and thus $a_t \equiv 0$), then we simply refer to $\msf R^{0,\eps}$ as kinetic Brownian motion. Whenever there is no risk of confusion, we omit the explicit dependence on $\gamma \geq 0$ and $\eps > 0$ and simply write $\msf R$ for the reference process. 

For any $t\in [0,1]$, let $\mu_t^a \defeq {\rm Law}(X_t^a,V_t^a)$ where $(X_t^a,V_t^a)_{t\in[0,1]}$ evolves according to \eqref{eq:fp_dyn}. Under appropriate conditions on $(a_t)_{t\in[0,1]}$ (see e.g., \cite{bogachev2022fokker}), it can be shown that the pair $(\mu_t^a,a_t)_{t\in[0,1]}$ satisfies the \emph{kinetic Fokker--Planck equation} (also called the Vlasov--Fokker--Planck equation), which is written as
\begin{align}\label{eq:kinetic_fp}\tag{$\msf{KFP}$}
\begin{split}
&\partial_t \mu_t^a + \langle \nabla_x \mu_t^a, v \rangle + \nabla_v \cdot(a_t \mu_t^a) = \frac{\eps}{2}\Delta_v \mu_t^a\,.
\end{split}
\end{align}

\section{Acceleration Matching}\label{sec:main}
In this section, we present our proposed algorithm for trajectory inference called Acceleration Matching (\texttt{AM}).

Our main technical observations and contributions appear in Section~\ref{sec:condacc}, where we construct a suitable interpolation process, the acceleration field corresponding to the second-order flow, and the Markovianized process. Our complete algorithm appears in Section~\ref{sec:loss}, which includes a detailed description of our simulation-free training objective. Finally, we describe the inference step of our method in Section~\ref{sec:gen_trajec}.

\subsection{Conditional acceleration fields and Markovianization}\label{sec:condacc}
Recall our second-order dynamics of interest given by {\eqref{eq:fp_dyn}}.
Our goal is to define a particular family of acceleration fields that are easy to learn from positional data in the form of marginals $\{(t_j, \rho_{t_j})\}_{j=0}^J$. Importantly, we want the resulting dynamics to yield marginals that agree with \eqref{eq:marg_constraints}, i.e., $\mathrm{Law} (X_{t_j}^a) = \rho_{t_j}$. 

To this end, take an arbitrary coupling $\pi \in \Pi(\rho_{t_0},\ldots,\rho_{t_J})$; for simplicity, consider the independent coupling of the marginals. Let $(X_t,V_t)_{t\in\ccint{0,1}} \sim \msf R$ be a kinetic (underdamped) Brownian motion for some fixed $\gamma \geq 0$ and $\varepsilon >0$. We define our interpolation process $\mathsf P^I \in \cP({\cal C}^{2d})$ in two steps: 
\begin{enumerate}
    \item For $\bm{t} = (t_0,\ldots,t_J) \in [0,1]^{J+1}$, sample $X_{\bm t}  \sim \pi$ 
    \item Sample the remainder of the path from the regular conditional distribution $\mathsf R_{|\bm t}$,
\end{enumerate}
where a formal definition of $\msf R_{|\bm t}$ is given in Appendix~\ref{sec:supp_section_main}. Mathematically, our interpolation process is written as
\begin{align}\label{eq:interp-process}
    {\msf P}^I = \int \msf R_{|\bm t}(\bm x,\cdot)\,\pi(\rmd \bm x)\,,
\end{align}
where $\msf P^I \in {\cal P}({\cal C}^{2d})$. As we assume query access to $\pi$ (at least in the case of independent marginals), it remains to efficiently  sample from our conditional distribution $\msf R_{|\bm t}$. Thankfully, it turns out that this remaining step corresponds to sampling from a suitable Gaussian distribution. Indeed, our next result shows that if the {conditional distribution of $V_0$ given $X_0$} is Gaussian, then $V_{\bm t}$ given $X_{\bm t}$ is also Gaussian. 
\begin{lem}
  \label{lem:condi_v_t}
    Assume that if $(X_0,V_0) \sim \msf R_0$, the conditional distribution of $V_0$ given $X_0$ is Gaussian almost surely with conditional mean $x_0 \mapsto m_0(x_0)$ and covariance matrix $x_0 \mapsto \Sigma_0(x_0)$. Then, the conditional distribution of $V_{\bm t}$ given $X_{\bm t}$ is also Gaussian with explicit conditional mean and covariance; see \eqref{eq:mean_cov_cond}.
\end{lem}
In brief, given only position knots $X_{\bm t}$, we can easily sample velocities $V_{\bm t}\,|\, X_{\bm t}$ in closed-form. The proof is an easy consequence of Lemma~\ref{lem:fund} and is deferred to Appendix~\ref{sec:proof-lemma-refl}.

The next step is to complete the sampling procedure for the \emph{entire path} $(X_t, V_t)_{t\in[0,1]}$ given the $J+1$ knots $(X_{\bm t}, V_{\bm t})$ in phase space. In the case of standard Brownian motion, this would be immediate, as it reduces to simulating independent Brownian bridges. Remarkably, an analogous characterization holds for the kinetic damped Brownian motion in phase space. For the remainder of this section, we focus on the case $\gamma = 0$, deferring the general case to the appendix.
\begin{lem}[Kinetic Brownian bridges]\label{lem:bridge_process}
Fix $t_j < t_{j+1} \in [0,1]$, $x_j,x_{j+1} \in \R^d$, $v_{j}, v_{j+1} \in \R^d$, and let $\msf R$ be kinetic Brownian motion with $\eps > 0$. {In the interval $[t_j,t_{j+1})$, the bridge of the reference process, i.e., the conditional distribution of $(X_t,V_t)_{t\in\ccint{t_j,t_{j+1}}}$ given  $(X_{t_{j}},V_{t_{j}}) = (x_j, v_j)$ and $(X_{t_{j+1}},V_{t_{j+1}}) = (x_{j+1}, v_{j+1})$} is the unique weak solution starting from $(X_{t_j},V_{t_j})$ to
\begin{align}\label{eq:bridge_proc}
        \dd X_t = V_t\dd t\,, \quad \dd V_t = a^j_t(X_t,V_t; X_{t_{j+1}},V_{t_{j+1}})\dd t + \sqrt{\varepsilon} \dd{ B_t^j}\,,
\end{align}
on $[t_j,t_{j+1})$, where $\{(B_t^{k})_{t \in\ccint{t_k,t_{k+1}}} \, :\, k \in \{0,\ldots,J-1\}\}$ are independent $d$-dimensional Brownian motions and  for any $x,v,x_{j+1},v_{j+1} \in\rset^d$ and $t \in\ccint{t_j,t_{j+1}}$,
\begin{align}\label{eq:a_jt}
        a_t^{j}(x,v;x_{j+1},v_{j+1}) &=
    \frac{6(x_{j+1} - x)}{(t_{j+1}-t)^2} - \frac{2(v_{j+1} + 2v)}{(t_{j+1}-t)}\,.
\end{align}
\end{lem}
The proof of Lemma~\ref{lem:bridge_process} is based on an application of Doob's $h$-transform on phase space; see Appendix~\ref{sec:proof-bridge-process}. An important corollary of Lemma~\ref{lem:bridge_process} is that the non-Markovian path measure $\msf P^I$ can now be represented as a solution to a sequence of stochastic differential equations: given $(X_{\bm t}^I, V_{\bm t}^I)$, Lemma~\ref{lem:bridge_process} shows that for any  $j\in\{0,\ldots,J-1\}$, $(X^I_{t},V_t^I)_{t\in[t_j,t_{j+1}]}$ is the unique weak solution starting from $(X_{t_j},V_{t_j})$ to\looseness-1
\begin{align}\label{conditional_acceleration_SDE}
\dd X^I_t = V^I_t \dd t\,, \quad \dd V^I_t = a_t^j(X^I_t,V^I_t ;X^I_{t_{j+1}},V^I_{t_{j+1}})\dd t + \sqrt{\eps}\dd B_t^j\,,
\end{align}
over the interval $t\in[t_j,t_{j+1}]$.

Based on this observation, it becomes transparent that $(X_t^I,V_t^I)_{t\in\ccint{0,1}}$ need not be Markov. However, we can find a diffusion process $(X_t^{M},V_{t}^M)_{t \in \ccint{0,1}}$ sharing the same marginals, replacing the acceleration field in \eqref{conditional_acceleration_SDE} with a suitable Markovianization that will be easy to learn using a parametric model. This is formalized by our next result.
\begin{thm}\label{thm:main}
  {Let} $(X_t^I,V_t^I)_{t\in\ccint{0,1}} \sim \msf P^I$ be an interpolated process associated with a coupling $\pi \in \Pi(\rho_{t_0},\ldots,\rho_{t_J})$ and a kinetic Brownian distribution $\msf R$ for $\eps >0$ fixed.
Consider an acceleration field {$(a_t^M)_{t\in\ccint{0,1}}$} satisfying almost surely
\begin{align*}
    a_t^{M}(X_{t}^I,V_t^I) \defeq \sum_{j=0}^{J-1} \bm{1}_{[t_j,t_{j+1})}(t) \, \bbE[a_t^j(X_t^I,V_t^I;X_{t_{j+1}}^I,V_{t_{j+1}}^I)|(X_t^I,V_t^I)]\,,
\end{align*}
where $a_t^j$ is given by \eqref{eq:a_jt}.  {In addition, suppose that $a_t^M$ is Lipschitz in all variables for any $t\in\ccint{0,1}$}. For $(X_0^M,V_0^M) \sim \mathsf P^I_0$, consider $(X_t^M,V_t^M)_{t\in\ccint{0,1}}$ the diffusion process solution over $ [0,1]$
\begin{align*}
    \dd X_t^M = V_t^M \dd t\,, \quad \dd V_t^M = a_t^{M}(X_t^M,V_t^M)\dd t + \sqrt{\eps}\dd B_t\,.
\end{align*}
Then, for any $t \in [0,1]$, it holds that $\mathrm{Law}(X_t^M,V_t^M)=\mathrm{Law}(X_t^I,V_t^I)$, and therefore the positional marginals match the prescribed constraints.
\end{thm}
In brief, Theorem~\ref{thm:main} states that the marginal laws of the original non-Markovian interpolant can be recovered through a suitably defined acceleration field. In particular, under the stated regularity assumptions, this projection preserves the phase-space distribution at each time while simultaneously  defining a Markovian law for the generated trajectories. This construction should be viewed as a direct \emph{lifted} analogue of, e.g., flow matching \citep{albergo2022building,lipman2022flow,liu2022flow} in the two-marginal setting, corresponding to $J=1$ in our framework.  In particular, we do not claim that \texttt{AM} learns the true underlying distribution of the data, but that the ideal acceleration field will match the marginals that we have access to at training time; we leave open the possibility to explore the theoretical assumptions, consistency of our method as the number of samples grows, and incorporation of approximation error.

\begin{remark}
We now briefly comment on related works that incorporate phase-space dynamics in their learning paradigm. The first is \cite{chen2024generative}, that also derives phase-space kinetic bridges. Their resulting acceleration field, which can be recovered as a particular case of the kinetic bridge acceleration derived above in Equation~\eqref{eq:a_jt}, is then approximated by regression. A separate work by \cite{chen2025tada} extends this framework by augmenting the space to arbitrary $N$-variables which recovers Flow Matching for $N=1 $ and phase space dynamics for $N=2$.  Importantly, neither work addresses our setting where multiple marginals are present, and the goal of their work is to develop more efficient samplers for generative modeling of a single distribution of image data.
\cite{yue2025oatfm} present how to fine-tune a Flow Matching model by minimizing the acceleration transport cost between the base and the target distribution. Finally, we mention \cite{brigati2025kinetic} which develops a theory of optimal transport under such kinetic (underdamped) dynamics, and further connections to this paper are interesting to explore.
\end{remark}

\subsection{Loss function and training}\label{sec:loss}
We now describe a simple algorithm for learning the Markovian acceleration field from Theorem~\ref{thm:main}. Let $a_\theta$ denote a parametrized neural network such that $a_\theta : [0,1] \times \R^{2d} \to \R^d$. Given that $a_t^M$ is a (piecewise) conditional expectation, it is natural to consider the following population loss
\begin{align}\label{eq:pop_loss}
\begin{split}
        \cL_{\texttt{AM}}(a_\theta) = \sum_{j=0}^{J-1}\int_{t_j}^{t_{j+1}} \bbE\| a_\theta(t,X_t,V_t) - a_t^j(X_t,V_t;X_{t_{j+1}},V_{t_{j+1}})\|^2
        \dd t\,,
\end{split}
\end{align}
where the expectation is under $\mathsf P^I$. At each time $t \in (t_j,t_{j+1})$, it is easy to see that $a_t^M$ is the unique minimizer to the conditional squared-error loss, and so it is omitted.\footnote{In practice, we use cutoffs and observe times in the window $[t_{j}+\delta, t_{j+1}-\delta]$.}

We now describe a practical implementation of \eqref{eq:pop_loss} in more detail (see also Algorithm~\ref{alg:stocacc}).
Recall that access to samples of positional data $X_{t_j}$ is readily available from the marginals $\{\rho_{t_j}\}_{j=0}^{J}$.
We then sample $V_{\bm t}\mid X_{\bm t}$ efficiently by Lemma~\ref{lem:condi_v_t} where, {under the reference law $\mathsf R$,} we take ${\rm Law}(V_0\mid X_0) = \gauss(0,\sigma_v^2I)$ for $\sigma_v^2 >0$ for simplicity.  We sample $t \sim {\rm Unif}(0,1)$ and find the $j^{\rm th}$-interval such that $t \in [t_j, t_{j+1})$. Also by Gaussian conditioning formulae, we can sample the kinetic Brownian bridge at the sampled time $t$; see Appendix~\ref{sec:bridge_sampling}. Given all of these ingredients, we can now simply regress our acceleration field onto the explicit {bridge acceleration target} given by \eqref{eq:a_jt}.

We stress that unlike prior work, our loss does not merely decompose into a learning problem on disjoint intervals, as the velocities used to train $a_\theta$ are jointly sampled across time conditioned on all positional knots. This allows us to learn smooth dynamics purely from the formulation of the learning problem instead of explicitly imposing spline-like dynamics into the training pipeline.

\begin{remark}(On the choice of coupling.)
The flexibility of our framework allows us to choose any coupling we want so long as it is used consistently throughout the process. In other words, we require the same coupling when training the acceleration field as when we learn the initial velocity law. With this in mind, we investigate the use of the pairwise OT coupling in addition to the independent coupling, as used in prior work.
\end{remark}

\begin{algorithm}[t]
\caption{Training algorithm for Acceleration Matching}
\label{alg:stocacc}
{
\begin{algorithmic}[1]
\REQUIRE Times $0=t_0<\cdots<t_J=1$; marginals $\{\rho_{t_j}\}^J_{j=0}$; coupling $\pi\in\Pi(\rho_{t_0},\ldots,\rho_{t_J})$
\STATE \textbf{Initialize} network $a_\theta:[0,1]\times\mathbb{R}^{2d}\to\mathbb{R}^d$
\WHILE{not converged}
    \STATE Sample $t$ uniformly from ${\rm Unif}(0,1)$
    \STATE Find $j\in\{0,\ldots,J-1\}$ such that $t_j\leq t<t_{j+1}$
    \STATE Sample all positional knots $\bm x\sim\pi$
    \STATE Sample all velocity knots $\bm v\mid\bm x$ 
    \STATE Sample $(x_t,v_t)$ conditioned on $(x_{t_j},v_{t_j})$ and $(x_{t_{j+1}},v_{t_{j+1}})$
    \STATE Compute $\widehat{\mathcal L}_{\mathtt{AM}}(a_\theta)=\|a_\theta(t,x_t,v_t)-a_t^j(x_t,v_t;x_{t_{j+1}},v_{t_{j+1}})\|^2$
    \STATE Update $\theta\gets\texttt{optim}(\theta,\nabla_\theta\widehat{\mathcal L}_{\mathtt{AM}}(a_\theta))$
\ENDWHILE
\STATE \textbf{Return} $a_\theta$
\end{algorithmic}
}
\end{algorithm}

\subsection{Generating trajectories}\label{sec:gen_trajec}
To generate trajectories, we draw $X_0\sim\rho_0$ and require access to the interpolant's initial law of $V_0\mid X_0$.
Although the initial joint law can be approximated by drawing all position knots and their conditional Gaussian velocities, this requires access to the training marginals as well as the coupling $\pi$ at inference time. To circumvent this requirement at {inference time} and amortize the cost, we instead learn a sampler $V_0\sim q_\phi(\cdot\mid X_0)$, where $q_\phi$ approximates the initial law of interest.

While conditional samplers can be difficult to learn at times, our formulation allows us to {simplify} the learning process substantially. Indeed, when $\pi$ is taken to be the independent coupling, we describe in Appendix~\ref{sec:cond_sampling} how our ideal conditional sampler can be written as
\begin{align*}
    \operatorname{Law}(V_0\mid X_0=x_0) = \operatorname{Law}(U + \zeta)\,,
\end{align*}
where $U = \sum_{j=1}^J \alpha_j X_{t_j}$ where $(\alpha_0,\ldots,\alpha_J)$ are deterministic coefficients, with $\alpha_0=-\sum_{j=1}^J\alpha_j$, and $\sigma_{\rm post}^2$ is a posterior variance available in closed-form, and $\zeta \sim {\gauss}(\alpha_0x_0, \sigma_{\rm post}^2 I)$ independently of $U$. The coefficients and posterior variance follow from Gaussian conditioning principles combined with the fact that during training we have that $X_0 \perp (X_{t_1},\ldots,X_{t_J})$. In turn, it suffices to learn a single \emph{unconditional} sampler for $\operatorname{Law}(U)$, which is doable via e.g., a Gaussian mixture model or a normalizing flow. As an example, Figure~\ref{fig:gom_cond-init-vel} illustrates the performance of using a GMM to learn the initial velocity sampler for the GoM dataset; for more examples and an ablation study on the number of samples required to learn the initial velocity sampler via GMMs, see Appendix~\ref{sec:cond_sampling}.

\begin{figure}[t]
    \centering
    \includegraphics[width=\linewidth]{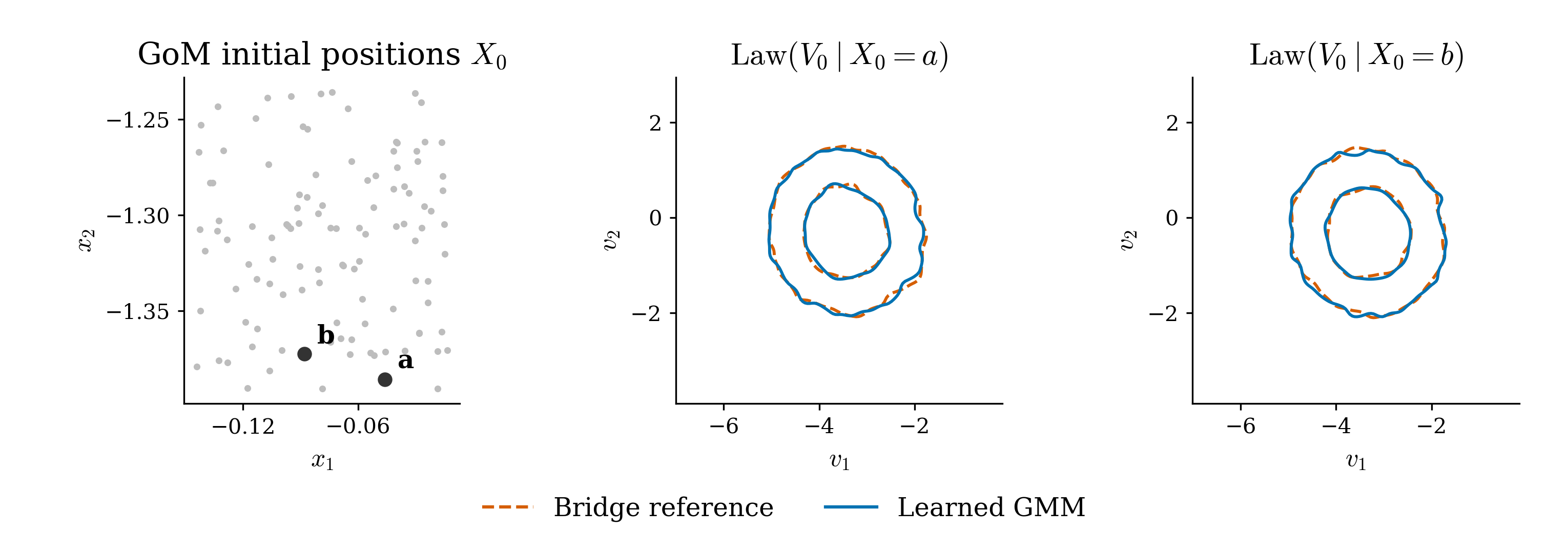}
    \caption{GoM conditional initial velocity sampler for two random source points.}
    \label{fig:gom_cond-init-vel}
\end{figure}

{When using pairwise OT couplings, the dependence between $X_0$ and the later position knots generally prevents a clean reduction. In this case, we learn the initial-velocity sampler} using conditional normalizing flows (NF) by obtaining initial position-velocity pairs via Lemma~\ref{lem:condi_v_t}; see Appendix~\ref{sec:cond_sampling} for more information.

In practice, given an initial position $X_0$, we then sample $V_0|X_0$ from one of the methods above (and the initial training data therefore does not need to be stored). Given the initial position and velocity, we then integrate the learned phase-space SDE\looseness-1
\begin{align*}
    \dd X_t = V_t \dd t \,,\quad 
    \dd V_t = a_\theta(t,X_t,V_t)\dd t + \sqrt{\eps}\dd B_t\,.
\end{align*}
In practice, we simulate the above SDE using a simple Euler--Maruyama discretization (see Algorithm~\ref{alg:em_sampler}) though other discretizations are certainly possible.
\begin{algorithm}[t]
\caption{Trajectory generation via Acceleration Matching}
\label{alg:em_sampler}
\resizebox{\linewidth}{!}{
\begin{minipage}{\linewidth}
\begin{algorithmic}[1]
\REQUIRE Learned conditional sampler $\widehat{q}_\phi$; learned acceleration
$\widehat a_\theta$; noise scale $\varepsilon$; number of steps $K$\looseness-1
\STATE Initialize $x_0 \sim \rho_0$, draw $v_0\sim \widehat{q}_\phi(\cdot\mid x_0)$
\STATE Set $h=1/K$
\FOR{$k=0,\ldots,K-1$}
    \STATE Update position $ x_{(k+1)h}
        \gets
        x_{kh} + h\, v_{kh}$
    \STATE Draw $\xi_k \sim \mathcal{N}(0,I)$
    \STATE Update velocity $v_{(k+1)h}
        \gets
        v_{kh}
        + h \, \widehat a_\theta(kh, x_{kh}, v_{kh})
        + \sqrt{\varepsilon h}\,\xi_k$
\ENDFOR
\STATE \textbf{Return} $(x_{kh})_{k=0}^{K}$
\end{algorithmic}
\end{minipage}
}
\end{algorithm}

\section{Numerical experiments}\label{sec:results}
We now present a set of experiments comparing \texttt{AM} to other state-of-the-art methods for the task of trajectory inference. Our main comparisons are with 3MSBM \citep{theodoropoulos2025momentum} (which is also based on a lifting to phase space) and simulation-free flow-matching methods such as MMFM~\citep{rohbeck2025modeling}, OT-CFM~\citep{tong2023improving,pooladian2023multisample}, and OT-MFM~\citep{kapusniak2024metric}.

On low-dimensional data, we benchmark against the official \href{https://github.com/Genentech/MMFM}{MMFM repository}, with all default parameters, which we call ``MMFM (prior)''. For both low- and high-dimensional datasets, we implement MMFM and CFM (cubic and linear interpolation, respectively) with constant variance per the official code (note that this induces Gaussian smoothing at the marginal knots $\rho_{t_j}$ rather than matching the marginals exactly; see Section~\ref{sec:related_work}). These methods always use pairwise OT couplings for their dynamics. Our implementation of OT-CFM uses linear interpolants on consecutive intervals, whereas MMFM uses cubic interpolants through the coupled knots. We use the published 3MSBM $W_2$ results, which do not report standard deviations (we were not able to independently verify their results). For the remaining baselines, we will report the results from the corresponding papers, denoted by $\text{OT-CFM}^*$ and $\text{OT-MFM}^*$   and will be highlighted in light pink for clarity.

As in prior work, we consider datasets of the form $\{(t_j,\widehat\rho_{t_j})\}_{j=0}^J$
where each $\widehat\rho_{t_j}$ is an empirical measure consisting of $n_{t_j}$ atoms in $\R^d$. For each dataset, we train on a subset of the available marginals and evaluate whether the generated trajectories match held-out marginals. 
To evaluate our methodology, we assess trajectory inference by measuring our ability to interpolate on held-out marginal data, and include visualizations of the generated paths.

\subsection{Low-dimensional datasets}
We first evaluate \texttt{AM} on two low-dimensional datasets widely used in the literature: ocean-current trajectories in the Gulf of Mexico (GoM), and Lotka--Volterra (LV) dynamics~\citep{shen2024multi}. We additionally consider data given by a (two-dimensional) nonlinear damped oscillator (NDO) with second-order dynamics. Concretely, this last setting fixes a potential $U :\R^2 \to \R$ and asks to learn dynamics given by
\begin{align}\label{eq:ndo}
    \dd X_t = V_t\dd t\,, \quad \dd V_t = (-\nabla U(X_t) - V_t)\dd t + \sqrt{\eps}\dd B_t\,.
\end{align}
For precise details on the training data, see  Appendix~\ref{app:oscillator_app}.
For the GoM and LV datasets, $J=8$ (thus nine total marginals) and we train on the even-time marginals, holding out the odd times $(t_1,t_3,t_5,t_7)$. For the NDO dataset, we simulate \eqref{eq:ndo} until terminal time $T=10$, and obtain $41$ equispaced (unpaired) snapshots to learn the dynamics, and evaluate on $20$ equispaced held-out marginals.
At inference time, we sample initial positions from $\widehat\rho_{t_0}$, simulate trajectories forward in time, and compare the generated samples with both training and held-out marginals using the $2$-Wasserstein distance for GoM and LV, and the $1$-Wasserstein distance for NDO.
In Figures~\ref{fig:traj_gom} and~\ref{fig:traj_lv}, we denote the held-out marginals by $t_j^*$.

\begin{table}[t]
\centering
\small
\caption{{Average held-out errors for GoM, LV, and NDO.} 3MSBM$^{*}$ values are taken directly from \cite{theodoropoulos2025momentum}.}
\label{tab:low-dimensional-w2-300-official}
\label{tab:oscillator-results}
\setlength{\tabcolsep}{7pt}
\begin{tabular}{lccc}
{Method} & {GoM ($W_2$)} & {LV ($W_2$)} & {NDO ($W_1$)} \\
\midrule
\rowcolor{quotedliterature}
3MSBM$^{*}$ & 0.135 & 0.230 & --- \\
\midrule
{OT-CFM} & $0.2594\stddev{0.0120}$ & $0.6304\stddev{0.0714}$ & $0.8637\mathbin{\pm}0.1647$ \\
{MMFM (prior)} & $0.209\stddev{0.012}$ & $0.427\stddev{0.031}$ & $2.0736\mathbin{\pm}1.2608$ \\
{MMFM} & $0.1702\stddev{0.0060}$ & $0.5213\stddev{0.0386}$ & $0.6350\mathbin{\pm}0.1960$ \\
AM (indep.) & $0.1409\stddev{0.0076}$ & $0.4073\stddev{0.0745}$ & $0.3158\mathbin{\pm}0.0908$
\end{tabular}
\end{table}
Table~\ref{tab:low-dimensional-w2-300-official} summarizes our results for the low-dimensional datasets; see Table~\ref{tab:low-dimensional-training-errors} in Appendix~\ref{app:experimental-details} for the training errors. Under the reported configurations, our method has lower mean errors than MMFM and OT-CFM, though the reported 3MSBM holdout errors remain lower. Figure~\ref{fig:traj_lv} contains representative generated trajectories for LV, with GoM trajectories deferred to Figure~\ref{fig:traj_gom} in Appendix~\ref{app:gom}. For NDO, \texttt{AM} displays the lowest hold-out error between all tested methods; see Figure~\ref{fig:oscillator-trajectories} for sample trajectories.

\begin{figure}[!htbp]
    \centering    \includegraphics[width=0.85\linewidth]{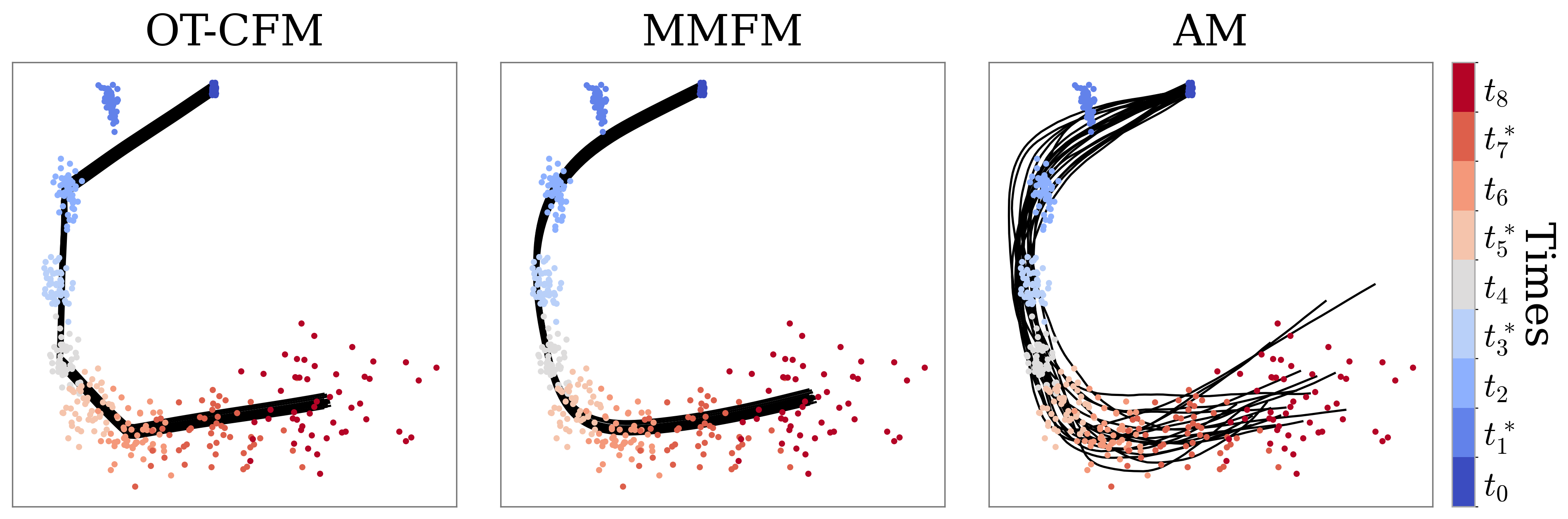}
    \caption{Visualization of trajectories generated for the {LV} dataset; $t^*_j$ denotes held-out marginals. }
    \label{fig:traj_lv}
\end{figure}

\begin{figure}[H]
\centering
\includegraphics[width=0.95\linewidth]{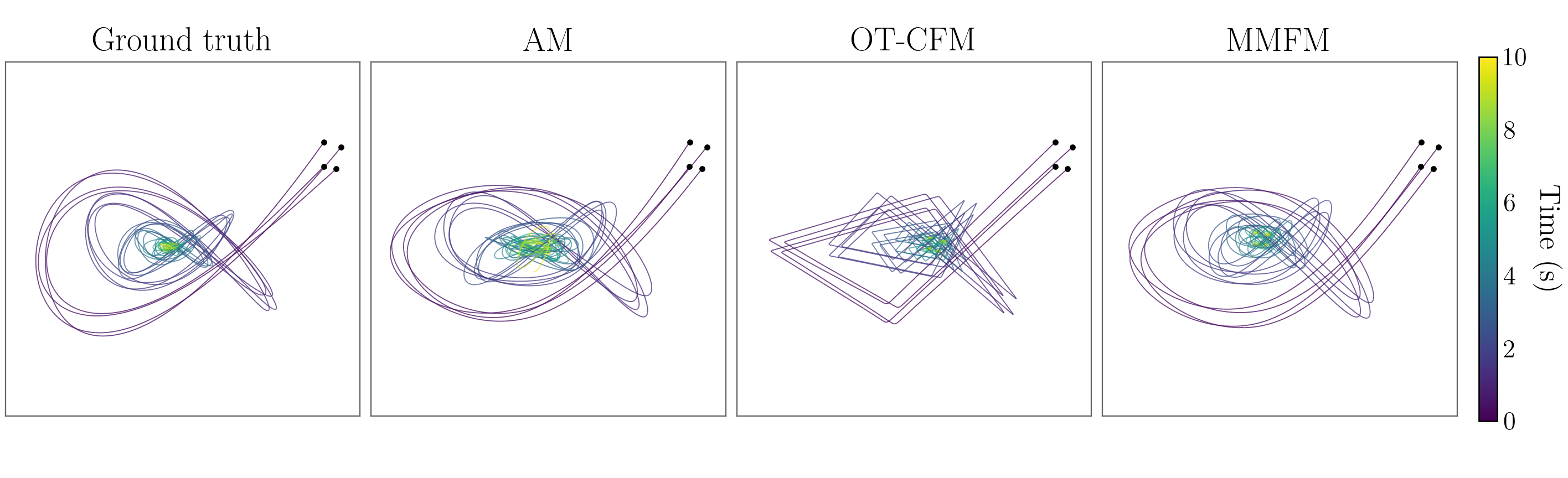}
\caption{Generated trajectories for the NDO dataset.}
\label{fig:oscillator-trajectories}
\end{figure}

\subsection{High-dimensional datasets}
We next evaluate \texttt{AM} on higher-dimensional datasets that arise from single-cell biology. Following prior work, we use PCA representations of developing embryoid bodies (EB)~\citep{moon2019eb} and CITE-seq data (CITE)~\citep{burkhardtcite}, reducing the dimensionality to $5$ for EB5 and CITE5, $50$ for CITE50, and $100$ for EB100. For EB5, CITE5, and CITE50, we use a leave-one-marginal-out (LOO) protocol, always keeping the initial and final marginals in the training set, and report the average held-out $W_1$ error. For EB5, there are 5 marginals and we leave out one of $t_1, t_2,$ or $t_3$, and average the resulting errors; for CITE5 and CITE50, we leave out either $t_1$ or $t_2$, and average the resulting errors. For EB100, we instead train on $t_0,t_2,t_4$ and hold out $t_1,t_3$ together, averaging the two held-out errors.

The high-dimensional results are summarized in Table~\ref{tab:summary-high-dim}. On EB5, \texttt{AM} with OT coupling improves over MMFM and OT-CFM, while the quoted OT-MFM reports a slightly lower error. On EB100, OT-{AM} returns the lowest error, closely followed by AM (indep.) and {MMFM}. On CITE5, \texttt{AM} with OT coupling is comparable to MMFM, though OT-CFM returns the lowest error. The CITE50 setting remains challenging: \texttt{AM} with OT coupling is competitive with MMFM, but all three of our implementations lag behind the quoted flow-matching results.

\begin{table}[t]
\centering
\small
\setlength{\tabcolsep}{3pt}
\caption{
Holdout $W_1$ error for the single-cell datasets averaged across five seeds; lower is better. OT-CFM$^*$ and OT-MFM$^*$ were taken from \citet[Tables~4 and~6]{kapusniak2024metric}.}
\label{tab:summary-high-dim}
\label{tab:cell-coupling-comparison}
\resizebox{0.85\linewidth}{!}{%
\begin{tabular}{@{}lcccc@{}}
Method & EB5 & EB100 & CITE5 & CITE50 \\
\midrule
\rowcolor{quotedliterature}
OT-CFM$^{*}$ & $0.790\stddev{0.068}$ & -- & $0.882\stddev{0.058}$ & $38.756\stddev{0.398}$ \\
\rowcolor{quotedliterature}
OT-MFM$^{*}$ & $0.713\stddev{0.039}$ & -- & $0.724\stddev{0.070}$ & $36.394\stddev{1.886}$ \\
\midrule
OT-CFM & $0.839\mathbin{\pm}0.011$ & $13.210\mathbin{\pm}0.043$ & ${0.491}\mathbin{\pm}0.004$ & ${39.137}\mathbin{\pm}0.338$ \\
MMFM & $0.779\mathbin{\pm}0.025$ & $12.926\mathbin{\pm}0.052$ & $0.532\mathbin{\pm}0.008$ & $41.235\mathbin{\pm}0.417$ \\
\midrule
{}AM (OT) & $0.745\mathbin{\pm}0.029$ & $12.852\mathbin{\pm}0.055$ & $0.548\mathbin{\pm}0.011$ & $43.067\mathbin{\pm}0.133$ \\
AM (OT, $\gamma = 20$) & $0.827\mathbin{\pm}0.101$ & $14.470\mathbin{\pm}0.399$ & $0.521\mathbin{\pm}0.014$ & $40.736\mathbin{\pm}0.066$ \\
AM (indep.) & $0.802\mathbin{\pm}0.038$ & ${12.955}\mathbin{\pm}0.081$ & $0.571\mathbin{\pm}0.026$ & $46.098\mathbin{\pm}0.996$ \\
\end{tabular}%
}
\end{table}

\begin{figure}[!htbp]
    \centering
    \includegraphics[width=\linewidth]{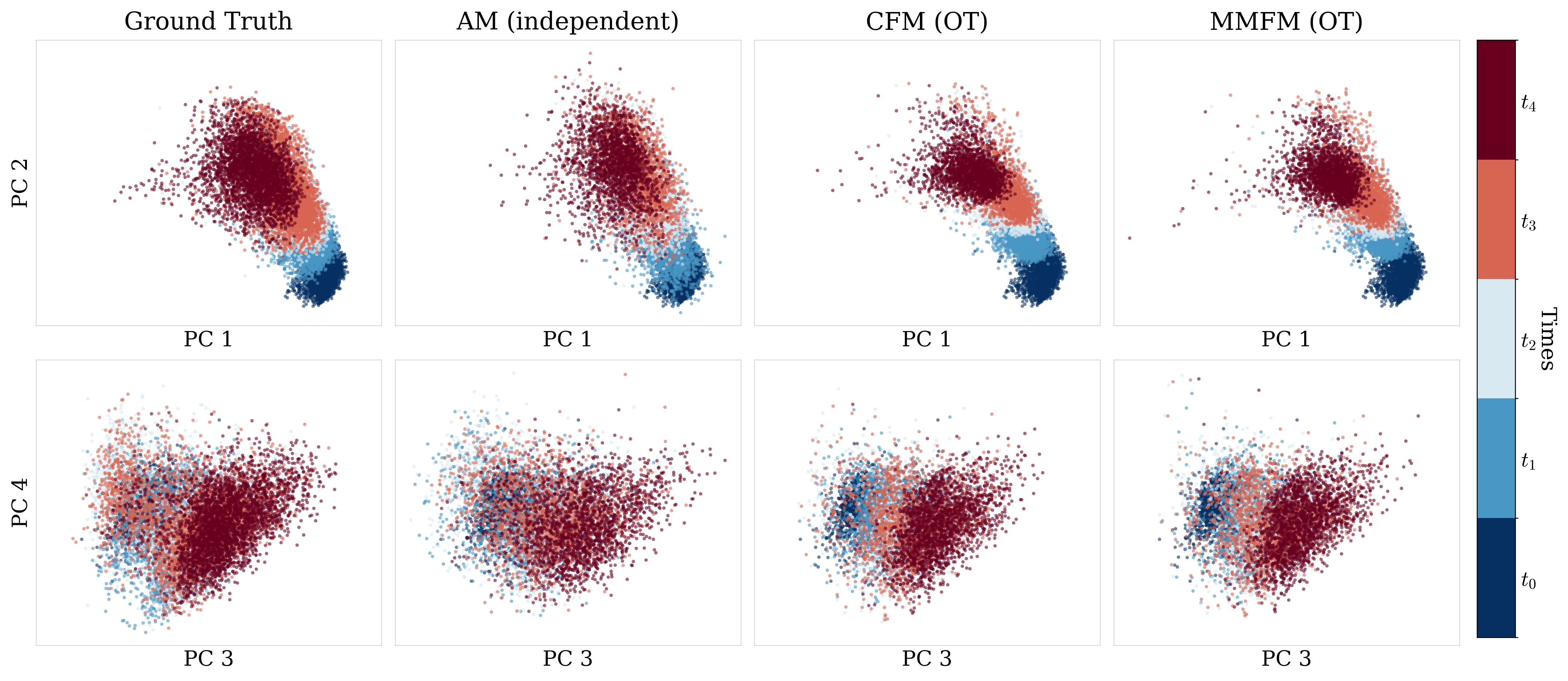}
    {
    \caption{Observed and generated marginals for EB5, with $t_1$ held out.}
    \label{fig:eb5-global-time-marginals}
    }
\end{figure}

\subsection{Related work}\label{sec:related_work}

\subsubsection*{Multi-sample flow matching and related algorithms}
Flow matching \citep{lipman2022flow,albergo2022building,liu2022flow} is a simulation-free approach for generating probability paths between two marginals.
Multi-marginal flow matching (MMFM) is a simulation-free approach proposed by \citet{rohbeck2025modeling}, which uses ideas reminiscent of other works that incorporate data-dependent couplings  \citep{tong2023improving,pooladian2023multisample,albergo2023stochastic} combined with using cubic splines.
Similar to the proposal by \citet{chewi2021fast}, they first compute pairwise optimal transport matchings between the marginals. For datasets of moderate size, pairwise matchings can be obtained using a standard optimal transport package \citep{flamary2021pot} where, writing $n = \max_j n_{t_j}$ as the maximum number of points per marginal, where we note that the total storage required scales like ${\cal O}(Jn^2)$ storage and the solvers have a runtime runtime of ${\cal O}(Jn^3)$. Then, given such a pairwise aligned coupling, they sample data $X_{0:J} \defeq (X_0,\ldots,X_J) \sim {\rm pairwise}(\rho_0,\ldots,\rho_J)$, and approximate the following velocity field\looseness-1
\begin{align*}
    \bbE[\partial_t \,{\rm spline}(t; X_{0:J}) + \sigma_t' \xi\,|\, {\rm spline}(t;X_{0:J}) + \sigma_t \xi = z]\,,
\end{align*}
where ${\rm spline}(t; X_{0:J})$ is a cubic spline between the drawn knots, $\sigma_t \geq 0$ is the noise coefficient, and $\xi \sim \cN(0,I)$. Upon learning the resulting vector field, they simulate the resulting ordinary differential equation to obtain the trajectories. \citet{rohbeck2025modeling} were also interested in learning conditional velocity fields via classifier-free guidance, which we do not pursue here.

The code provided by \citet{rohbeck2025modeling} mostly uses a constant variance, i.e., $\sigma_t' = 0$, and so our implementation of MMFM does the same. We note that this choice of constant variance induces positional marginals which are Gaussian-smoothed versions of the prescribed marginal knots.

Optimal transport conditional flow matching (OT-CFM) uses minibatch optimal transport couplings between the marginals when fitting a velocity field \citep{tong2023improving,pooladian2023multisample}. \cite{kapusniak2024metric} proposed OT-Metric Flow Matching (OT-MFM), which incorporates a Riemannian metric from the data in addition to learning a vector field. The Riemannian metrics take the form of kernel density estimators (e.g., radial basis functions). We also mention ALI-CFM, a two-stage method that learns neural interpolants through an adversarial objective~\citep{kviman2026adversarial}. Here, a discriminator distinguishes samples from the observed intermediate marginals from samples generated by the interpolant and once learned, the interpolant is held fixed, and a velocity field is trained to model the interpolant's velocity using conditional flow matching. %

\subsubsection*{Multi-marginal Schr\"odinger bridge methods}
 Many neural-network-based algorithms for solving the trajectory inference problem are based on multi-marginal Schr\"odinger bridges. For instance, multi-marginal stochastic flow matching (MMSFM) \citep{lee2025multi} is a variant of MMFM that incorporates some ideas from \citet{shi2024diffusion}, {and uses OT couplings between pairs or triplets of marginals}. Schr\"odinger bridge with iterative reference refinement (SBIRR) \citep{shen2024multi} requires some degree of trajectory simulation as part of the algorithm and has non-smooth positional trajectories as they use standard SDEs, not lifted ones, for generating trajectories.

The work of \citet{theodoropoulos2025momentum} also passes to phase space (which they refer to as ``momentum''), where their goal is to learn a Markovian acceleration field. Apart from this similarity, our proposed method is quite distinct from theirs. They regress on a different conditional acceleration field, one which depends on \emph{several} future positions in a non-trivial way. Their proposed conditional acceleration field is defined as
\begin{align*}
    &a_t^{\rm{3MSBM}}((x_t,v_t)\,|\,\{\bm x_{j+1} \, : \, t_{j+1} \geq t\}) = C_1^j(t)(x_t - \bm x_{j+1}) + C_2^j(t) v_t + C_3^j(t)\textstyle\sum_{\ell=j+1}^J \lambda_\ell \bm x_\ell\,,
\end{align*}
for $t \in [t_j, t_{j+1})$, where $\{\bm x_j\}_{j=0}^J$ denotes the sampled knots, $\{\{C_k^j(t)\}_{k=1}^3\}_{j=0}^J$ are time-varying coefficients corresponding to each segment, and $\{\lambda_j\}_{j=0}^J$ are also computed. {Their training pipeline iteratively propagates conditional dynamics to obtain conditional velocities. This allows them to avoid backpropagating through trajectories, though simulation is still costly. Our construction instead directly samples the velocity knots by simple Gaussian conditioning, followed by direct sampling of intermediate bridge states and the explicit target in \eqref{eq:a_jt}.}

\section{Conclusion}
We propose Acceleration Matching as a {principled} algorithm for trajectory inference. To learn {continuously differentiable positional trajectories}, we lift the learning problem to phase space, with the goal of learning an appropriate Markovianized acceleration field. This perspective yields {an explicit regression objective for learning dynamics with multiple marginal constraints, as all arguments and targets can be directly computed via simple Gaussian conditioning. Our AM has the flexibility of avoiding pairwise-OT preprocessing by using the {independent} coupling, and we use a conditional initializer to amortize the velocities sampled at inference time.} A limitation of our method is its sensitivity to hyperparameter choices, which may require problem-specific tuning to achieve strong performance. Several natural directions remain beyond the scope of this work, such as investigating the role of the reference process, improving overall algorithmic complexity, extending our framework to instance-specific Riemannian manifolds~\citep{chen2023flow,yim2023se,haviv2024wasserstein}, quantifying discretization guarantees in the spirit of~\citet{chen2022sampling} and~\citet{conforti2025kl}, and applying our method to more computationally intensive settings where either $J$ or $d$ (or both) are large.

\section*{Acknowledgments}
AAP is grateful for financial support from the Institute for Foundations of Data Science at Yale University and to Ricardo Baptista for helpful comments on an earlier version of this draft.

\section*{{AI use statement}}
AI tools were used to assist with developing  the methodology, writing code, \LaTeX{} formatting, and in polishing writing and mathematical details. We take full responsibility for the final content of the paper.

\bibliography{main}

\newpage
\appendix
\section{Supplement of Section~\ref{sec:main}}
\label{sec:supp_section_main}
We begin with some additional notation.

We denote by $(X_t)_{t\in[0,1]}$ the canonical process, i.e., $X_t(h) = h_t \in \R^d$ for any $h \in {\cal C}^d$ and any $t \in [0, 1]$. For any $t$, $\msf Q_t$ denotes the $t$-marginal $(X_t)_\sharp \msf Q$ for $\msf Q \in \mathcal{P}({\cal C}^d)$. By a slight abuse of notation, we denote by $(X_t,V_t)_{t\in[0,1]}$ the canonical process on the phase space $\rset^{2d}$, i.e., $X_t(h) = \mathrm{proj}_x(h_t)$ and $V_t(h) = \mathrm{proj}_v(h_t)$ for any $h \in {\cal C}^{2d}$ and any $t \in [0, 1]$. Here, $\mathrm{proj}_x$ and $\mathrm{proj}_v$ denote the projections onto the first and last $d$ components, respectively.

We also recall the dynamics associated with kinetic underdamped Brownian motion, denoted $\msf R = \msf R^{\gamma,\eps}$
\begin{align}\label{eq:ref_proc}
    \begin{cases}
        \dd X_t = V_t\dd t \\
        \dd V_t = -\gamma V_t \dd t + \sqrt{\varepsilon} \dd B_t\,.
    \end{cases}
\end{align}
For additional shorthand, let $Z_t \defeq (X_t,V_t)$ for $t \in [0,1]$.

We now provide a formal definition of $\msf R_{|\bm t}$; this is a Markov kernel on $(\rset^{d})^{J+1} \times \mathcal{B}(\mathcal{C}^{2d})$ satisfying for any measurable sets $A \in \mathcal{B}(\mathcal{C}^{2d})$ and $B \in \mathcal{B}((\rset^d)^{J+1})$, $\msf R(A\cap \{X_{\bm t} \in B\}) = \PE[\1_{B}(X_{\bm t}) \msf R_{|\bm t}(X_{\bm t},A)]$. Recall the definition of $\mathsf P^I$ from \eqref{eq:interp-process}.

\subsection{Proof of Lemma~\ref{lem:condi_v_t}}
\label{sec:proof-lemma-refl}
\begin{lemma}\label{lem:fund}
  Conditionally on $(X_0,V_0)$, the reference process \eqref{eq:ref_proc} is a Gaussian process
  with mean and covariance functions given by, for $s,t \in \ccint{0,1}$ and $s \leq t$,
  \begin{equation*}
    m_Z(t) = \begin{pmatrix}
X_0 + V_0  a_t \\[6pt]
V_0 b_t
\end{pmatrix}  =  A_t Z_0 \,,
    \end{equation*}
where $a_t = \gamma^{-1}(1-\rme^{-\gamma t})$, $b_t = \rme^{-\gamma t}$, and
\begin{align*}
    {A}_t &= \begin{pmatrix}
        1 & a_t \\
        0 & b_t
      \end{pmatrix} \otimes \Idd \,.
\end{align*}
Moreover, it holds that for $s\leq t$, $\Gamma_Z(s,t) \defeq {\rm Cov}(Z_s,Z_t\,|\,Z_0)$ is given by
\begin{align*}
     \Gamma_Z(s,t) &=
     \begin{pmatrix}
\mathrm{Var}(X_s) + a_{t-s}\mathrm{Cov}(X_s,V_s)  & b_{t-s} \mathrm{Cov}(X_s,V_s)  \\
\mathrm{Cov}(X_s,V_s) + a_{t-s}  \mathrm{Var}(V_s)   & b_{t-s}  \mathrm{Var}(V_s)
     \end{pmatrix}
\otimes \Idd \,,
\end{align*}
where
\begin{align*}
    \mathrm{Var}(V_s) &= (\eps /(2 \gamma))( 1  - \rme^{-2 \gamma s} )\,, \\
    \mathrm{Var}(X_s)& =   (\varepsilon/\gamma^2) \left[ s - 2 a_s + (1/(2\gamma))(1-\rme^{-2\gamma s})\right]\,, \\
  \mathrm{Cov}(X_s,V_s) &  =  (\eps/2)a_s^2\,.
\end{align*}
In particular, if $\gamma = 0$, then the above simplifies to
  \begin{align*}
    m_Z(t) = \begin{pmatrix}
        X_0 + V_0 t \\
        V_0
    \end{pmatrix}
    = A_t Z_0{\,,} \qquad A_t =
    \begin{pmatrix}
        1 & t \\
        0 & 1
    \end{pmatrix} \otimes \Idd{\,,}
\end{align*}
and
\begin{align*}
    \Gamma_Z(s,t) =
\varepsilon     \begin{pmatrix}
  s^2(3t-s)/6 & s^2/2\\
  s^2/2 +s(t-s) & s
    \end{pmatrix}
    \otimes \Idd{\,.}
\end{align*}
\end{lemma}

\begin{proof}
First, we only show the result for $d=1$ since given $(X_0,V_0)$ all components are independent and therefore the result tensorizes. Note that for $s,t \in \ccint{0,1}$, we can readily solve the SDE for the velocity term to yield
\begin{align*}
    V_t = \rme^{-\gamma(t-s)}V_s + \eta_{s,t}^V \defeq \rme^{-\gamma(t-s)}V_s + \sqrt{\eps}\int_s^t \rme^{-\gamma(t-u)}\dd B_u\,,
\end{align*}
and for $s=0$, we have that $V_t = \rme^{-\gamma t}V_0 + \sqrt{\eps}\int_0^t \rme^{-\gamma(t-u)} \dd B_u$. Solving for the position variable then yields
\begin{align*}
    X_t &= X_s + \int_s^t V_u \dd u \\
    &= X_s + a_{t-s}V_s + \eta_{s,t}^X\\
    &\defeq X_s + a_{t-s}V_s+ \sqrt{\varepsilon} \int_s^t \int_{s}^{u} \rme^{-\gamma(u - r)} \dd B_{r} \dd u\,,
\end{align*}
where in the second step we used that $\int_s^t \rme^{-\gamma(u-s)}\dd u = a_{t-s}$. Therefore, as a solution of a linear SDE, $(X_t,V_t)_{t\in \ccint{0,1}}$ is indeed a Gaussian process given initial datum $(X_0,V_0)$. We now proceed and compute the remaining explicit expressions, which make heavy use of the It\^o isometry and the Markov property of the Brownian motion.

To start, we compute the double integral as
\begin{align*}
  \int_s^t \int_s^u \rme^{-\gamma(u-r)} \dd B_r \dd u = \int_s^t a_{t-r}\dd B_r\,,
\end{align*}
which further simplifies the expression for $X_t$, which will be useful later.

We readily compute the means of the Gaussian process as
\begin{align*}
    \bbE[V_t\,|\,(X_0,V_0)] = \rme^{-\gamma t} V_0\,, \quad \bbE[X_t\,|\,(X_0,V_0)] = X_0 + a_tV_0\,,
\end{align*}
which implies the representation $m_Z(t) = A_t Z_0$ and the accompanying matrix $A_t$.

Conditionally on $(X_0,V_0)$, define the centered variables
\begin{align*}
    \overline{V}_s = \sqrt{\eps}\int_0^s \rme^{-\gamma(s-u)}\dd B_u\,, \quad \overline{X}_s = \sqrt{\eps}\int_0^s a_{s-u}\dd B_u\,,
\end{align*}
and then we readily compute
\begin{align*}
    &{\rm Var}(V_s) = \bb E\Bigl[\overline{V}_s^2\Bigr] = \eps \int_0^s \rme^{-2\gamma(s-u)}\dd u = \frac{\eps}{2\gamma}(1-\rme^{-2\gamma s})\,,
\end{align*}
and similarly
\begin{align*}
{\rm Var}(X_s) = \bb E\Bigl[\overline{X}_s^2\Bigr] = \eps \int_0^s a_{s-u}^2 \dd u  &= \frac{\eps}{\gamma^2}\int_0^s (1 - 2\rme^{-\gamma r} + \rme^{-2\gamma r}) \dd r \\
&= \frac{\eps}{\gamma^2}\bigl(s - 2a_s + (2\gamma)^{-1}(1-\rme^{-2\gamma s}) \bigr)\,,
\end{align*}
and also
\begin{align*}
    {\rm Cov}(X_s,V_s) = \eps \int_0^s \rme^{-\gamma(s-u)} a_{s-u} \dd u = \frac{\eps}{2}a_s^2\,.
\end{align*}
We now compute the covariance between times $s \leq t$. Note that $V_t$ can be written in terms of $V_s$ and an additional Brownian motion term that is independent of $V_s$ (and similarly for $X_s$). Due to this independence, we can compute the covariance between times as
\begin{align*}
    {\rm Cov}(V_s,V_t) = {\rm Cov}(V_s, \rme^{-\gamma (t-s)}V_s + \eta_{s,t}^V) = {\rm Cov}(V_s, \rme^{-\gamma (t-s)}V_s) = \rme^{-\gamma (t-s)}{\rm Var}(V_s)\,,
\end{align*}
and similarly
\begin{align*}
    & {\rm Cov}(X_s,V_t) = {\rm Cov}(X_s, \rme^{-\gamma (t-s)}V_s) = \rme^{-\gamma(t-s)}{\rm Cov}(X_s,V_s)\,, \\
    &{\rm Cov}(V_s,X_t) = {\rm Cov}(V_s, X_s + a_{t-s}V_s) = {\rm Cov}(X_s,V_s) + a_{t-s}{\rm Var}(V_s)\,, \\
    &{\rm Cov}(X_s,X_t) = {\rm Cov}(X_s, X_s + a_{t-s}V_s) = {\rm Var}(X_s) + a_{t-s}{\rm Cov}(X_s,V_s)\,,
\end{align*}
which concludes the proof.
\end{proof}

Recall our notation conventions $0=t_0<t_1<\cdots<t_J= 1$ with $\bm{t} = (t_0,\ldots,t_J)$. Additionally define $\bm a, \bm b \in \R^{J+1}$ as
\begin{align*}
    \bm a \defeq a_{\bm t} \defeq (a_{t_0}, \ldots, a_{t_J}) \,, \quad
    \bm b \defeq b_{\bm t} \defeq (\rme^{-\gamma t_0}, \ldots, \rme^{-\gamma t_J})\,,
\end{align*}
as well as $\bm q \defeq (\bm a,\bm b) \in \R^{2(J+1)}$. Additionally, recall
\begin{align*}
V_{\bm{t}}
=
\begin{pmatrix}
V_0&
V_{t_1}&
\hdots&
V_{t_J}
\end{pmatrix}^{\top}
\in\mathbb R^{(J+1)d}{\,,}
\qquad
X_{\bm{t}}
=
\begin{pmatrix}
X_{0}&
X_{t_1}&
\hdots&
X_{t_J}
\end{pmatrix}^{\top}
\in\mathbb R^{(J+1)d}{\,.}
\end{align*}

\begin{lemma}\label{lem:condi_v_t_detailed}
Let $X_0 = x_0 \in \R^d$ be fixed, and assume that
\begin{align*}
    V_0 \mid X_0 = x_0 \sim {\gauss}(m_0(x_0), \Sigma_0(x_0))\,.
\end{align*}
Then conditionally on $X_0 = x_0$, $(X_{\bm t}, V_{\bm t})$ is (jointly) Gaussian with mean $\bm m = (m_{\bm X}, m_{\bm V})$, where the components are
\begin{align}
    m_{\bm X} \defeq \bm 1_{J+1} \otimes x_0 + \bm a \otimes m_0(x_0)\,,\quad m_{\bm V} = \bm b \otimes m_0(x_0)\,,
\end{align}
and covariance $\bm{\Sigma} \in \R^{2(J+1)d  \times 2(J+1)d}$ given by
\begin{align*}
    \bm{\Sigma} \defeq
    \begin{pmatrix}
\Sigma_{\bm X} &\Sigma_{\bm{XV}}\\
\Sigma_{\bm{VX}}&\Sigma_{\bm{V}}
\end{pmatrix} \defeq
({\bm{q}} {\bm{q}}^{\top} \otimes \Sigma_0(x_0))  + \begin{pmatrix}
K_{\bm X} &K_{\bm{VX}}^\top\\
K_{\bm{VX}} &K_{\bm V}
\end{pmatrix} \otimes I_d\,,
\end{align*}
where for $i,k \in \{0,\ldots,J\}$ we define
\begin{align*}
  [K_{\bm X}]_{ik} & = \mathrm{Var}(X_{t_i\wedge t_k}) + a_{|t_i-t_k|} \mathrm{Cov}(X_{t_i\wedge t_k}, V_{t_i\wedge t_k})\,, \\
[K_{\bm V}]_{ik}
  &= \rme^{-\gamma |t_i-t_k|} \mathrm{Var}(V_{t_i\wedge t_k})\,, \\
[K_{\bm{VX}}]_{ik}
  &    =              \begin{cases}
                    \mathrm{Cov}(X_{t_i},V_{t_i})
+
a_{t_k-t_i}\mathrm{Var}(V_{t_i}) & i \leq k\,, \\
\rme^{-\gamma(t_i-t_k)} \mathrm{Cov}(X_{t_k},V_{t_k}) & i >k \,.
                  \end{cases}
\end{align*}
Moreover, $V_{\bm t} \mid X_{\bm t}$ is Gaussian with mean and covariance given by
\begin{align}\label{eq:mean_cov_cond}
\begin{split}
m_{\bm V\mid \bm X} &= {m}_{\bm V} +
\Sigma_{\bm{VX}}\Sigma_{\bm X}^{\dagger}
(X_{\bm{t}}- m_{\bm X})\,, \\
\Sigma_{\bm V\mid \bm X}
&=
\Sigma_{\bm{V}}
-
\Sigma_{\bm{VX}}\Sigma_{\bm{X}}^\dagger\Sigma_{\bm{XV}}\,,
\end{split}
\end{align}
where $\Sigma_{\bm{X}}^{\dagger}$ denotes the Moore--Penrose pseudoinverse of $\Sigma_{\bm{X}}$.
\end{lemma}
\begin{proof}
By Lemma~\ref{lem:fund}, given $(X_0,V_0)$, the vector $( X_{\bm{t}},V_{\bm{t}})$ is jointly
Gaussian, so all that remains is to compute the mean and covariance.

Conditional on $X_0 = x_0$, we have by the representations in Lemma~\ref{lem:fund} that
\begin{align}\label{eq:lemma_eq}
    X_{t_i} = x_0 + a_{t_i} V_0 + \eta_i^X\,, \quad V_{t_i} = b_{t_i} V_0 + \eta_i^V\,,
\end{align}
where we recall that we defined $b_t = \rme^{-\gamma t}$. We further write the random variable $V_0\mid X_0 = x_0$ as $V_0 = m_0(x_0) + \xi_0$ for $\xi_0 \sim {\gauss}(0,\Sigma_0(x_0))$. Taking expectations leads to the definitions of $m_{\bm X}$ and $m_{\bm V}$ since $\xi_0$, $\eta_i^X$, and $\eta_i^V$ have mean zero. Then rewriting \eqref{eq:lemma_eq}, we have
\begin{align*}
    (X_{\bm t}, V_{\bm t}) = (m_{\bm X}, m_{\bm V}) + (\bm q \otimes I_d)\xi_0 + (\eta_{\bm t}^X, \eta_{\bm t}^V)\,,
\end{align*}
where, to compute the covariance, we separate the two sources of noise (which are independent from one another). We directly compute that
\begin{align*}
    {\rm Cov}((\bm q \otimes I_d)\xi_0) = (\bm q \otimes I_d){\rm Cov}(\xi_0)(\bm q \otimes I_d)^\top = (\bm q \bm q^\top)\otimes \Sigma_0(x_0)\,,
\end{align*}
and for ${\rm Cov}(\eta_{\bm t}^X, \eta_{\bm t}^V)$, we appeal to the existing formulae in Lemma~\ref{lem:fund}. The remaining conditional Gaussian formulas follow from standard computations that we omit.
\end{proof}

We end this section by computing a pseudoinverse of $\Sigma_{\bm{X}}$.
\begin{lemma}
  \label{lem:pseudo_inverse}
Let $\gamma \geq 0, \eps > 0$ and consider the reference process $\mathsf R^{\gamma,\eps}$, and recall the notation and setting of Lemma~\ref{lem:condi_v_t_detailed}. Write $L_{\bm X} \in \R^{J\times J}$ such that $[L_{\bm X}]_{ik} = [K_{\bm X}]_{ik}$ for $i,k \in \{1,\ldots,J\}$. Then $L_{\bm X}$ is positive definite, and it holds that
\begin{align}\label{eq:cov_pinv}
    (\Sigma_{\bm X})^\dagger = \begin{pmatrix}
        0 & 0 \\
        0 & \Theta^{-1}
    \end{pmatrix}\,,
\end{align}
where, for $\tilde {\bm a} = (a_{t_1}, \ldots, a_{t_J}) \in \R^J$, $\Theta$ is defined as
\begin{align*}
    \Theta = \tilde{\bm a} \tilde{\bm a}^\top \otimes \Sigma_0(x_0) + L_{\bm X}\otimes I_d\,,
\end{align*}
and its inverse is given by
\begin{align}\label{eq:theta_t_inv}
    \Theta^{-1} = L_{\bm X}^{-1} \otimes I_d - (L_{\bm X}^{-1} \tilde{\bm a}\tilde{\bm a}^\top L_{\bm X}^{-1})\otimes [\Sigma_0(x_0)(I_d + c_\star \Sigma_0(x_0))^{-1}]{\,,}
\end{align}
for  $c_\star = \langle \tilde{\bm a}, L_{\bm X}^{-1} \tilde{\bm a}\rangle$. In particular, if $\Sigma_0(x_0)$ is invertible, then
\begin{align*}
    \Theta^{-1} = L_{\bm X}^{-1} \otimes I_d - (L_{\bm X}^{-1} \tilde{\bm a}\tilde{\bm a}^\top L_{\bm X}^{-1})\otimes (\Sigma_0(x_0)^{-1} + c_\star I_d)^{-1}\,.
\end{align*}
\end{lemma}
\begin{proof}
We first prove that $L_{\bm X}$ is positive definite. As before, work over $\R$ since the $d$-dimensional covariance tensorizes with $I_d$.

From the representation in \eqref{eq:lemma_eq}, the centered positional noise at time $t_i$ is
\begin{align*}
    \eta_i^X = \sqrt{\eps}\int_0^{t_i} a_{t_i-u}\dd B_u\,,
\end{align*}
where $a_r = \gamma^{-1}(1-\rme^{-\gamma r})$ for $\gamma>0$, with the convention $a_r=r$ when $\gamma=0$. Hence, for any $c=(c_1,\ldots,c_J)\in\R^J$,
\begin{align*}
        \langle c, L_{\bm X} c\rangle
    =
    \eps\int_0^{1}
    \Bigl(
        \sum_{i=1}^J c_i a_{t_i-u}\bm 1_{\{u\leq t_i\}}
    \Bigr)^2
    \dd u \geq 0\,.
\end{align*}
We proceed by contradiction: if $\langle c, L_{\bm X} c\rangle = 0$ for $c \neq 0$, then the integrand would be zero for almost every $u$. Now let $m$ be the largest index such that $c_m\neq 0$, then on the interval $(t_{m-1},t_m)$ the terms corresponding to  $i<m$ will vanish since $t_i < u$, whereas all terms with $m < i$ have zero coefficient by definition of $m$ being the largest index. Thus the integrand reduces to $(c_m a_{t_m-u})^2$, which is strictly positive. This yields a contradiction, so $L_{\bm X}$ is positive definite and thus invertible.

The block form of $(\Sigma_{\bm X})^\dagger$ follows because $t_0=0$ makes the first row and column of $\Sigma_{\bm X}$ vanish, while the remaining block is $\Theta$. Since $L_{\bm X}$ is positive definite and $\Sigma_0(x_0)$ is positive semidefinite, $\Theta$ is invertible. The displayed expression for $\Theta^{-1}$ is then the Sherman--Morrison--Woodbury formula applied to $\Theta=L_{\bm X}\otimes I_d + \tilde{\bm a}\tilde{\bm a}^\top\otimes \Sigma_0(x_0)$.
\end{proof}

\subsection{Proof of Lemma~\ref{lem:bridge_process}}\label{sec:proof-bridge-process}
\begin{lemma}
For $\gamma \geq 0, \eps > 0$, consider  kinetic underdamped Brownian motion $\mathsf R^{\gamma,\eps}$, and for $t_j, t_{j+1} \in [0,1]$ with $t_j < t_{j+1}$, fix endpoints $(x_j,v_j)$ and $(x_{j+1},v_{j+1})$ respectively. Then the bridge of the reference process conditioned on $(X_{t_j},V_{t_j}) = (x_j,v_j)$ and $(X_{t_{j+1}},V_{t_{j+1}}) = (x_{j+1},v_{j+1})$ is the unique solution to
\begin{align}\label{eq:bridge_sde_general}
    \begin{cases}
        \dd X_t = V_t\dd t \\
        \dd V_t = (\eps\nabla_v \log p_{t_{j+1}-t}(X_t,V_t,x_{j+1},v_{j+1}) - \gamma V_t)\dd t + \sqrt{\varepsilon} \dd B_t\,,
    \end{cases}
\end{align}
on $t \in [t_j,t_{j+1})$, where $p_{t_{j+1} - t}(x,v;x_{j+1},v_{j+1})$ is the transition density of the reference process.
\end{lemma}

\begin{proof}
Since both the $\msf R$-bridge and the solution of \eqref{eq:bridge_sde_general} are Markov processes, it suffices to show that their infinitesimal generators coincide. If we denote by $p_t((x,v),(x',v'))$ the transition density of the reference process $\msf R$, then the marginal flow of the bridge at time $t \in [t_j, t_{j+1})$ is
\begin{equation*}
\mu^{\msf br}_t(x,v) \propto
{p_{t-t_j}((x_j,v_j),(x,v))p_{t_{j+1}-t}((x,v),(x_{j+1},v_{j+1}))}\,.
\end{equation*}
Our goal is to understand the evolution of $\mu^{\msf{br}}_t$ in the weak sense. To this end, consider now a smooth test function $\phi:[t_{j},t_{j+1}]\times \mathbb{R}^{2d}\to\mathbb{R}$ with compact support; we compute
\begin{align*}
    \partial_t \int \phi_t \mu_t^{\msf{br}} = \int (\partial_t \phi_t) \mu_t^{\msf{br}} + \phi_t (\partial_t p_{t-t_j})p_{t_{j+1}-t} + \phi_t p_{t-t_j}(\partial_t p_{t_{j+1}-t})\,.
\end{align*}
To continue, we require the infinitesimal generator of the kinetic (underdamped) reference process, which we recall is given by
\begin{align*}
    \cL_\gamma f(x,v) = \langle v, \nabla_x f(x,v)\rangle - \gamma \langle v, \nabla_v f(x,v)\rangle + \frac{\eps}{2}\Delta_{v}f(x,v){\,,}
\end{align*}
for a test function $f$. Now, we use the appropriate forward and backward Kolmogorov equations to see that
\begin{align*}
\partial_t p_{t-t_j} = \cL_\gamma^* p_{t-t_j}\,, \quad \text{and} \quad \partial_t p_{t_{j+1}-t} = -\cL_\gamma p_{t_{j+1}-t}{\,,}
\end{align*}
resulting in
\begin{align}
\begin{split}\label{eq_h_trans_1}
    \partial_t \int \phi_t \mu_t^{\msf{br}} &= \int (\partial_t \phi_t) \mu_t^{\msf{br}} + \phi_t ( \cL_\gamma^* p_{t-t_j})p_{t_{j+1}-t}  - \phi_t p_{t-t_j}(\cL_\gamma p_{t_{j+1}-t}) \\
    &= \int (\partial_t \phi_t) \mu_t^{\msf{br}} + \cL_\gamma (\phi_t p_{t_{j+1}-t}) p_{t-t_j}  - \phi_t p_{t-t_j}(\cL_\gamma p_{t_{j+1}-t})\,,
\end{split}
\end{align}
where the last line uses integration by parts with the adjoint $\cL_\gamma$.
Next, we use the commutation identity for two test functions \citep{bakry2014analysis}
\begin{equation*}
\begin{aligned}
\mathcal{L}_\gamma(g h) &=(\mathcal{L}_\gamma g) h  + g(\mathcal{L}_\gamma h) + \varepsilon \langle \nabla_{v} g, \nabla_v h\rangle{\,,}
\end{aligned}
\end{equation*}
to rewrite \eqref{eq_h_trans_1} as
\begin{align}\label{eq_h_trans_2}
\begin{split}
\partial_t \int \phi_t \mu_t^{\msf{br}} &=
\int \Bigl((\partial_t \phi_t)\mu_t^{\mathsf{br}} +  (\mathcal{L}_\gamma \phi_t) p_{t_{j+1}-t} p_{t-t_{j}} + \varepsilon\, \langle \nabla_v \phi_t, \nabla_v p_{t_{j+1}-t} \rangle p_{t - t_j}\Bigr) \\
 &= \int ( (\partial_t + \mathcal{L}_\gamma) \phi_t   + \varepsilon \, \langle \nabla_v \log p_{t_{j+1}-t}, \nabla_v \phi_t\rangle  )\mu_t^{\msf{br}} \,,
\end{split}
\end{align}
which implies that $\mu_t^{\msf{br}}$ is a weak solution to the Fokker--Planck equation associated with the generator
\begin{equation*}
\mathcal{L}^{\msf{br}}_{\gamma, t} := \cL_\gamma + \varepsilon  \langle \nabla_v \log p_{t_{j+1}-t}, \nabla_v \cdot \rangle\,.
\end{equation*}
\end{proof}

\begin{lemma}
Let $\mathsf R^{\gamma,\eps}$ be the kinetic underdamped Brownian motion reference process for $\gamma > 0$ and $\eps > 0$, and fix $T \in (0,1]$ and $z_T = (x_T,v_T) \in \R^{2d}$. Letting  $p_{T-t}((x,v);z_T)$ denote the transition density from time $t \in [0,T)$ to $T$ and $r \defeq T - t$, it holds that
\begin{align}\label{eq:damped_bridge_score}
    \eps \nabla_v \log p_{T-t}((x,v); z_T) =
    C_x(r;\gamma)(x_T - x - a_r v) + C_v(r;\gamma)(v_{T}- b_r v)\,,
\end{align}
where
\begin{align*}
    C_x(r;\gamma) &= \gamma^2(1-\rme^{-\gamma r})/D_r(\gamma)\,, \\
    C_v(r;\gamma) &= \gamma(\rme^{-2\gamma r} + 2 \gamma r \rme^{-\gamma r} - 1)/((1-\rme^{-\gamma r})D_r(\gamma))\,, \\
    D_r(\gamma) &= \gamma r(1 + \rme^{-\gamma r}) + 2\rme^{-\gamma r} - 2\,.
\end{align*}
In particular, $\gamma \to 0$ recovers \eqref{eq:a_jt}.
\end{lemma}
\begin{proof}
Writing $r = T-t$, it holds by Lemma~\ref{lem:fund} that, conditioning on $Z_t = (x,v)$
\begin{align*}
    (X_T,V_T) | (X_t,V_t) = (x,v) \sim {\gauss}\Bigl(\begin{pmatrix}
        x + a_rv \\
        b_r v
    \end{pmatrix},  Q_r\Bigr)\,,
\end{align*}
where $Q_r = \Gamma_Z(r,r)$ (and recall the notation $a_r = \gamma^{-1}(1-\rme^{-\gamma r})$ and $b_r = \rme^{-\gamma r}$ from Lemma~\ref{lem:fund}).
Thus up to a constant, we have that
\begin{align*}
&\log p_{T-t}((x,v),(x_T,v_T)) =\\
&\qquad -\frac{1}{2} \begin{pmatrix}
    x_T -
    (x + a_r v)\\
    v_T - b_r v
\end{pmatrix}^\top {Q_r^{-1}}\begin{pmatrix}
    x_T -
    (x + a_r v)\\
    v_T - b_r v
\end{pmatrix}  + {\rm const}\,.
\end{align*}
As the matrix $Q_r^{-1}$ is explicit (via block 2-by-2 inversion), one can verify through a tedious calculation that indeed \eqref{eq:damped_bridge_score} holds, and setting $\gamma = 0$ recovers \eqref{eq:a_jt}.
\end{proof}
\subsection{Proof of Theorem~\ref{thm:main}}\label{proof_main}
Our proof follows the standard argument that passes through weak solutions of the kinetic underdamped Fokker--Planck equation. {Throughout this proof, $\gamma=0$, as in Theorem~\ref{thm:main}.} Letting $f : \R^d \times \R^d \to \R$ be a smooth test function, our goal is to compute
\begin{align*}
    \int f \,\partial_t \mu_t^I = \partial_t \Bigl(\int f \dd \mu_t^{I}\Bigr) = \partial_t \bbE[f(X_t^I,V_t^I)]\,,
\end{align*}
where $\mu_t^I$ is the density of $\mathsf P^I_t$. To proceed, we apply It\^o's lemma to first arrive at
\begin{align*}
        \dd f &= \langle \nabla_x f, \dd X_t^I \rangle + \langle \nabla_v f, \dd V_t^I\rangle + \frac{1}{2}\langle {\rm d} V_t^I, \nabla_{v}^2 f \dd V_t^I \rangle \\
        &= \bigl(\langle \nabla_x f, V_t^I \rangle + \langle \nabla_v f, a_t^j \rangle
        - \gamma \langle\nabla_v f, V_t^I \rangle +  \frac{\eps}{2}\Delta_v f\bigr)\dd t  + \sqrt\eps \langle \nabla_v f, \dd B_t\rangle\,.
    \end{align*}
Taking expectations, as the last term has mean zero by the martingale property, yields
\begin{align*}
    \partial_t \bbE[f(X_t^I,V_t^I)] &= \bbE\Bigl[ \langle \nabla_x f(X_t^I,V_t^I), V_t^I \rangle + \langle \nabla_v f(X_t^I,V_t^I), a_t^j(X_t^I,V_t^I) \rangle\\
    &\phantom{=}\quad - \gamma \langle\nabla_vf(X_t^I, V_t^I), V_t^I\rangle  + \frac{\eps}{2}\Delta_vf(X_t^I, V_t^I)\Bigr]\,,
\end{align*}
where recall that $a_t^j(X_t^I,V_t^I) = a_t^j(X_t^I, V_t^I; X_{t_{j+1}}^I, V_{t_{j+1}}^I)$. By the tower property of conditional expectations, it holds that
\begin{align*}
    &\bbE[\langle a_t^j(X_t^I, V_t^I; X_{t_{j+1}}^I, V_{t_{j+1}}^I), \nabla_v f(X_t^I,V_t^I) \rangle ] \\
    &\quad = \bbE[\bbE[\langle a_t^j(X_t^I, V_t^I; X_{t_{j+1}}^I, V_{t_{j+1}}^I), \nabla_v f(X_t^I,V_t^I) \rangle \,|\,(X_t^I,V_t^I)]] \\
    &\quad = \bbE[\langle \bbE[a_t^j(X_t^I, V_t^I; X_{t_{j+1}}^I, V_{t_{j+1}}^I)\,|\,(X_t^I,V_t^I)], \nabla_v f(X_t^I,V_t^I) \rangle ] \\
    &\quad = \bbE[\langle {a_t^{M}}(X_t^I,V_t^I), \nabla_v f(X_t^I,V_t^I) \rangle ]\,,
\end{align*}
where we recall that for an interval $[t_j,t_{j+1})$
\begin{align*}
    (t,x,v)\mapsto a_t^{M}(x,v) \defeq \bbE[a_t^j(X_t^I, V_t^I;X_{t_{j+1}}^I, V_{t_{j+1}}^I)\,|\,(X_t^I,V_t^I) = (x,v)]\,.
\end{align*}
Therefore, we see that the weak solution to $\mu_t^I$ satisfies
\begin{align*}
    \partial_t \bbE[f(X_t^I,V_t^I)] &= \bbE\Bigl[ \langle \nabla_x f(X_t^I,V_t^I), V_t^I \rangle + \langle \nabla_v f(X_t^I,V_t^I), a_t^{M}(X_t^I,V_t^I)\rangle \\
    &\phantom{=}\quad -\gamma \langle \nabla_v f(X_t^I, V_t^I), V_t^I \rangle + \frac{\eps}{2}\Delta_vf(X_t^I, V_t^I)\Bigr]\,.
\end{align*}
Since $\mu_0^M = \mu_0^I$ and, by assumption, the acceleration field implies a unique solution to the kinetic Fokker--Planck equation, we conclude the proof.

\section{Algorithmic subroutine details}
\subsection{Learning an initial velocity sampler}\label{sec:cond_sampling}
Theorem~\ref{thm:main} requires that $(X_0^M,V_0^M) \sim \mathsf P^I_0$, meaning that the initial position and velocity need to be appropriately sampled from the law of the interpolant at time $t=0$. As mentioned in Section~\ref{sec:gen_trajec}, our goal is to learn a conditional sampler for ${\rm Law}(V_0\mid X_0)$. As we mentioned in Section~\ref{sec:gen_trajec}, we will now show how it suffices to merely learn an \emph{unconditional} sampler for this process under independent coupling, which is much easier.

First, recall that under the reference process (again, assuming $\gamma = 0$ for simplicity) we have that
\begin{align*}
    X_t - X_0 = t V_0 + \sqrt{\eps} \int_0^t (t-s) \dd B_s\,.
\end{align*}
Writing
\begin{align*}
    Y = \begin{pmatrix}
        X_{t_1} - X_0 \\
        \vdots \\
        X_{t_J} - X_0 \\
    \end{pmatrix}\,,\qquad
    \bm t_+ \defeq (t_1,\ldots,t_J)^\top\in\R^J\,,
\end{align*}
notice then if $V_0 \sim \gauss(0,\sigma_v^2I)$, $(V_0,Y)$ is jointly Gaussian with 
\begin{align*}
    \operatorname{Var}(V_0) &= \sigma_v^2 {I}\,,\quad\operatorname{Cov}(Y,V_0)= (\sigma_v^2\bm t_+)\otimes {I}\,,\quad \operatorname{Cov}(Y) = (\sigma_v^2\bm t_+\bm t_+^\top + \eps Q)\otimes {I}\,,
\end{align*}
where $Q$ is given by the positional covariance kernel in Lemma~\ref{lem:fund},
\begin{align*}
    Q_{ik} = \frac{(t_i\wedge t_k)^2[3(t_i\vee t_k)-(t_i\wedge t_k)]}{6} \quad (i,k\in\{1,\ldots,J\})\,.
\end{align*}
Writing $K \defeq \eps Q + \sigma_v^2\bm t_+\bm t_+^\top$ and $\bm c \defeq \sigma_v^2\bm t_+$, we have that
\begin{align*}
    V_0\mid Y \sim \gauss((\bm\alpha^\top\otimes {I})Y, \sigma_{\rm post}^2 {I})\,,
\end{align*}
where $\sigma_{\rm post}^2 \defeq \sigma_v^2 - \langle \bm{c}, K^{-1} \bm c\rangle$ and $\bm\alpha = K^{-1}\bm c$. Now, expanding the mean, we can write
\begin{align*}
    (\bm\alpha^\top\otimes {I})Y
    &= \alpha_0 X_0 + \sum_{j=1}^J\alpha_j X_{t_j}\,,\qquad
    \alpha_0\defeq-\sum_{j=1}^J\alpha_j\,.
\end{align*}
Note that the coefficients $\bm\alpha$ are entirely deterministic. Now, if we restrict our attention to the independent coupling, we have that $X_0 \perp (X_{t_1},\ldots,X_{t_J})$. Therefore, to learn a conditional sampler for inference time, it suffices to be able to sample from
\begin{align*}
    \operatorname{Law}(V_0\mid X_0=x_0) = \operatorname{Law}(U+\zeta)\,,
\end{align*}
where $U \defeq \sum_{j=1}^J\alpha_j X_{t_j}$ is always independent of $X_0$, and $\zeta\sim\gauss(\alpha_0x_0,\sigma_{\rm post}^2 {I})$ is independent of $U$. In particular, it suffices to learn an \emph{unconditional} sampler for $\operatorname{Law}(U)$ prior to the generation process.

We note that for other non-trivial choices of the interpolant coupling $\pi$ (e.g., pairwise OT couplings),
we can efficiently sample interpolants
\begin{align*}
    v_{0:J}\mid X_{0:J}=x_{0:J}\,,
\end{align*}
from the Gaussian conditioning formula~\eqref{eq:mean_cov_cond}. We then extract the initial position and velocity which correspond to index $j=0$; repeating this procedure yields a paired dataset ${\{(x_0^{(i)},v_0^{(i)})\}_{i=1}^n}$. Then, using this paired data, we can hope to learn a conditional sampler $q_\phi(\cdot\mid x_0)$ that samples from $V_0\mid X_0$.

\subsubsection{Visualization of the learned conditional sampler}
For some of the datasets, we now visualize our conditional velocity sampler against the ground-truth (which requires knowledge of the training data at inference time). We visualize the conditional sampler for LV, GoM, EB5 and CITE5 datasets for simplicity.

\begin{figure}[!htbp]
    \centering
    \includegraphics[width=\linewidth]{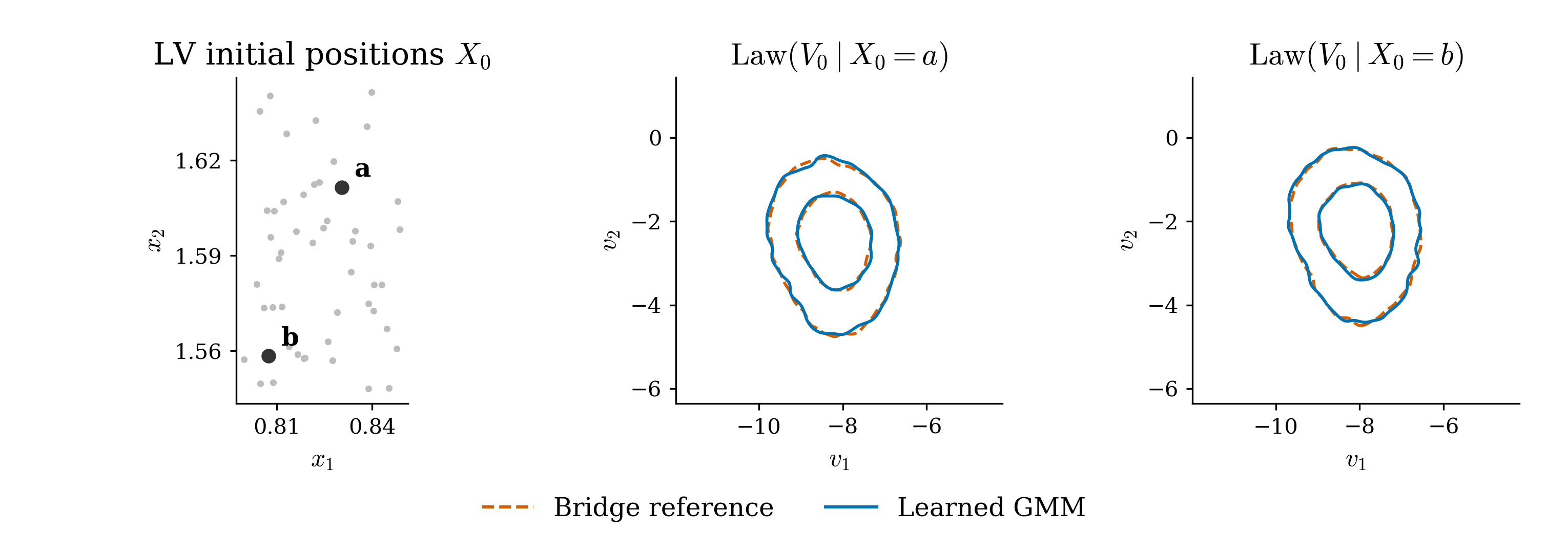}
    \caption{LV conditional initial velocity law at two random source points.}
    \label{fig:lv_cond-init-vel}
\end{figure}

\begin{figure}[htbp]
    \centering
    \begin{minipage}[t]{0.47\linewidth}
        \centering
        \includegraphics[width=\linewidth]{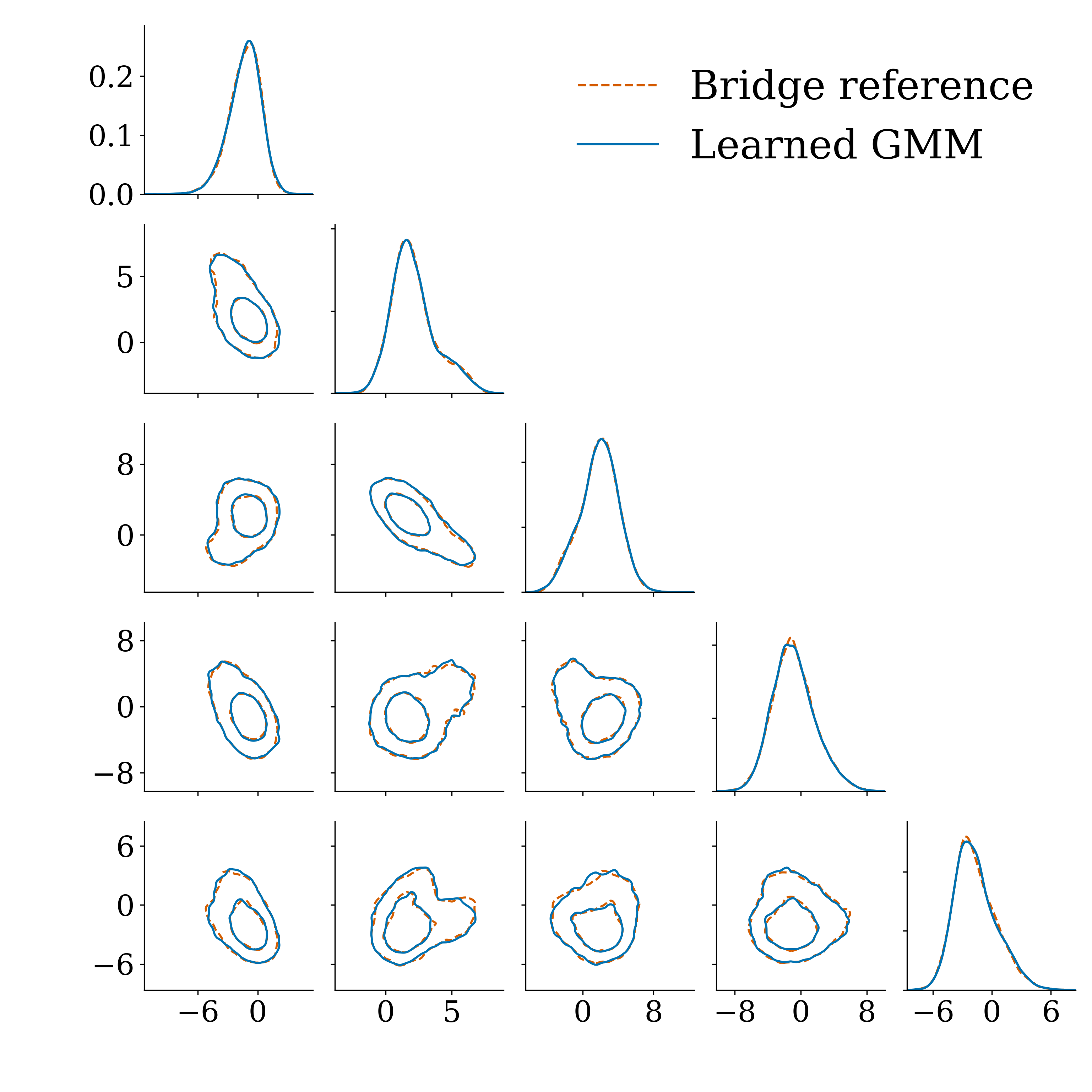}
    \end{minipage}
    \hfill
    \begin{minipage}[t]{0.47\linewidth}
        \centering
        \includegraphics[width=\linewidth]{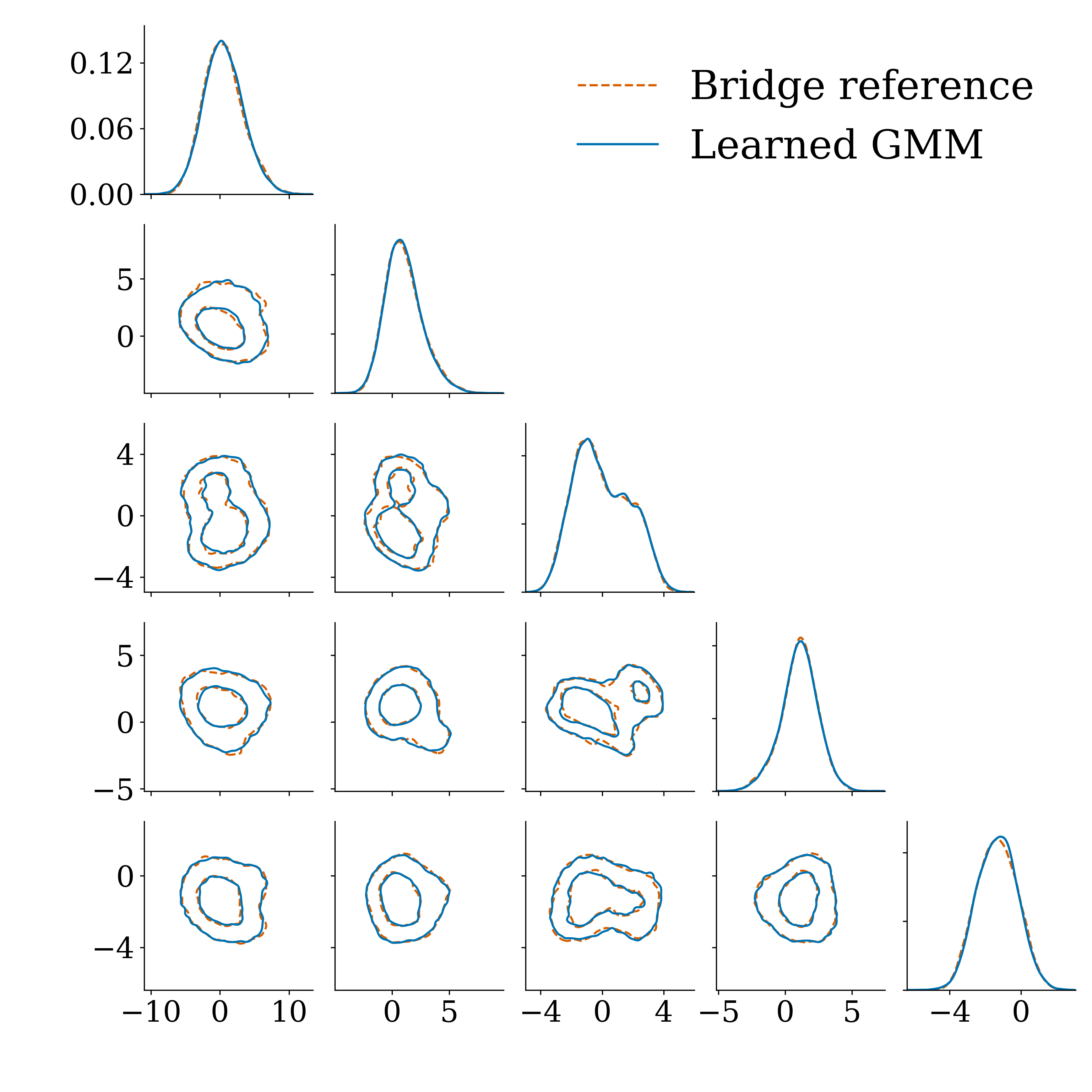}
    \end{minipage}
    \caption{Learned conditional initial velocity distributions for EB5 \textbf{(left)} and CITE5 \textbf{(right)} for a fixed random source point.}
    \label{fig:cell-initial-velocities}
\end{figure}

\subsubsection{Learning versus direct conditional samplers}
To determine the influence of the learned initial velocity sampler, we compare the performance of AM where we sample ${\rm Law}(V_0\mid X_0 = x_0)$ directly from the interpolant distribution using training data versus our learned sampler. Table~\ref{tab:independent-main-vs-direct} compares the held-out errors between these two methods for a given fixed acceleration field trained using Algorithm~\ref{alg:stocacc} (note that the initial velocity sampler does not play a role in the training of the acceleration field, so these effects can be distinguished). We notice that the differences in the results are negligible; sometimes learning the sampler performs better than using the direct interpolant distribution, and vice versa.
\begin{table}[!htbp]
\centering
\caption{Comparing a learned initial velocity sampler versus directly sampling the interpolant distribution using training data. We report the held-out errors using a fixed acceleration network and the independent coupling, only varying the samplers, over five seeds. We observe that the differences between the methods is negligible.}
\label{tab:independent-main-vs-direct}
\begin{tabular}{@{}cccc@{}}
{Sampler method} & Dataset & Learned sampler & Direct sampler \\
\midrule
\multirow{4}{*}{{GMM}}
& GoM ($W_2$)   & $0.1409 \pm 0.0076$  & $0.1367 \pm 0.0079$ \\
& LV ($W_2$)    & $0.4073 \pm 0.0745$  & $0.3999 \pm 0.0703$ \\
& EB5 ($W_1$)   & $0.8024 \pm 0.0381$  & $0.8034 \pm 0.0310$ \\
& CITE5 ($W_1$) & $0.5707 \pm 0.0263$  & $0.5615 \pm 0.0234$ \\
\midrule
\multirow{2}{*}{{NF}}
& EB100 ($W_1$)  & $12.9548 \pm 0.0812$ & $13.0005 \pm 0.0728$ \\
& CITE50 ($W_1$) & $46.0977 \pm 0.9959$ & $46.4323 \pm 0.7488$ \\
\end{tabular}
\end{table}

\subsubsection{Quality of learned samplers}
We present another ablation on the quality of the learned initial sampler and its impact on the generated path. Concretely, for EB5 and CITE5, we vary the number of training points in the dataset where more samples implies higher quality GMM estimates at the expense of longer fitting times, and otherwise use a fixed acceleration network. Table~\ref{tab:gmm-sample-budget} presents our findings. In particular, we see that an initial velocity sampler can be obtained by significantly fewer samples, and consequently faster, compared to our original choice of $20\,000$.

\begin{table}[htbp]
\centering
\setlength{\tabcolsep}{4pt}
\caption{Quality of learned initial velocity sampler as a function of the number of samples, and the resulting $W_1$ held-out error for a fixed trained acceleration network. }
\label{tab:gmm-sample-budget}
\begin{tabular}{@{}lcccc@{}}
& \multicolumn{2}{c}{EB5}
& \multicolumn{2}{c}{CITE5} \\
\cmidrule(lr){2-3}\cmidrule(lr){4-5}
Training samples
& Fitting time & Held-out $W_1$
& Fitting time & Held-out $W_1$ \\
\midrule
500
& $0.202 \pm 0.037$ & $0.8087 \pm 0.0459$
& $0.312 \pm 0.042$ & $0.5684 \pm 0.0284$ \\
5\,000
& $0.731 \pm 0.084$ & $0.7993 \pm 0.0327$
& $1.535 \pm 0.231$ & $0.5763 \pm 0.0216$ \\
10\,000
& $1.443 \pm 0.113$ & $0.7998 \pm 0.0320$
& $2.762 \pm 0.176$ & $0.5728 \pm 0.0218$ \\
{20\,000}
& $2.864 \pm 0.303$ & $0.8024 \pm 0.0381$
& $5.945 \pm 0.419$ & $0.5707 \pm 0.0263$ \\
40\,000
& $5.537 \pm 0.563$ & $0.8019 \pm 0.0309$
& $12.087 \pm 1.847$ & $0.5739 \pm 0.0241$ \\
\end{tabular}
\end{table}

\subsubsection{Initial velocity samplers for OT-AM}
We now isolate the effect of the number of updates required for learning conditional NF when using AM with pairwise OT couplings, where the acceleration field is kept fixed. We see that dropping from $20\,000$ iterations to $3\,000$ iterations results in minor performance changes (again, sometimes better, sometimes worse) with a near $7\times$ decrease in required compute cost.
\begin{table}[!htbp]
\centering
\setlength{\tabcolsep}{4pt}
\caption{OT-AM comparison using 3,000 versus 20,000 initial-velocity NF training updates.}
\label{tab:ot-am-nf-budget-comparison}
\begin{tabular}{@{}lccc@{}}
Dataset & NF updates & Held-out $W_1$ & NF training time $(s)$ \\
\midrule
\multirow{2}{*}{EB5}
& 3\,000 & $0.745 \pm 0.029$ & $82.673 \pm 0.159$ \\
& 20\,000 & $0.746 \pm 0.027$ & $579.570 \pm 2.416$ \\
\midrule
\multirow{2}{*}{EB100}
& 3\,000 & $12.852 \pm 0.055$ & $76.273 \pm 0.297$ \\
& 20\,000 & $13.119 \pm 0.088$ & $549.130 \pm 20.738$ \\
\midrule
\multirow{2}{*}{CITE5}
& 3\,000 & $0.548 \pm 0.011$ & $68.854 \pm 0.544$ \\
& 20\,000 & $0.541 \pm 0.005$ & $514.059 \pm 7.788$ \\
\midrule
\multirow{2}{*}{CITE50}
& 3\,000 & $43.067 \pm 0.133$ & $72.314 \pm 0.496$ \\
& 20\,000 & $42.546 \pm 0.126$ & $501.681 \pm 9.918$ \\
\end{tabular}
\end{table}

\subsection{Bridge sampling}\label{sec:bridge_sampling}
For a given $j \in \{0,1,\ldots,J-1\}$, define ${\cal I}_j = [t_{j},t_{j+1}]$ and $\delta_{j} = t_{j+1} - t_j$. We want to sample (a single) intermediate point (position and velocity) along the trajectory given two endpoint constraints
\begin{align*}
    (X_t,V_t) | (X_{t_{j+1}},V_{t_{j+1}}) = (x_{j+1},v_{j+1}), (X_{t_{j}},V_{t_{j}}) = (x_j,v_j)\,,
\end{align*}
where $t \in [t_j + c, t_{j+1} - c]$ for $0 < c \ll \delta_j/2$. Again, we can appeal to explicit formulae that arise in the study of Gaussian processes. For simplicity, we restrict our attention to the case $\gamma=0$.

Define
\begin{align*}
     Y_t = \begin{pmatrix}
        X_{t_j + t} \\
        V_{t_j + t}
    \end{pmatrix}\,,
\end{align*}
for $t \in [0,\delta_{j}]$, where the dynamics will again follow the reference process $\msf R$. Thus
\begin{align*}
     Y_t = m_t + \sqrt \varepsilon\, \xi_t = \begin{pmatrix}
        x_j + t v_j \\
        v_j
    \end{pmatrix} + \sqrt \varepsilon \begin{pmatrix}
        \int_0^t B_u \dd u \\
        B_t
    \end{pmatrix}\,.
\end{align*}
Here, the initialization pins down the left endpoint constraints and it remains to enforce the right endpoint. Since $(Y_t; Y_{\delta_{j}})$ is jointly Gaussian, we can readily compute ${\rm Cov}(\xi_t)$, ${\rm Cov}(\xi_{\delta_j})$ and ${\rm Cov}(\xi_t,\xi_{\delta_j})$ explicitly. By Gaussian conditioning, we have that $Y_t \, | \, Y_{\delta_j} \sim {\gauss}(m_{t|j}, C_{t|j})$ with
\begin{align*}
      m_{t|j} &= m_t + {\rm Cov}(\xi_t,\xi_{\delta_j}){\rm Cov}(\xi_{\delta_j})^{-1}(Y_{\delta_j} - m_{\delta_j})\,, \\
      C_{t|j} &= \varepsilon {\rm Cov}(\xi_t) - \varepsilon {\rm Cov}(\xi_t,\xi_{\delta_j}){\rm Cov}(\xi_{\delta_j})^{-1}{\rm Cov}(\xi_t,\xi_{\delta_j})^\top\,.
\end{align*}
To conclude, we simply substitute $Y_{\delta_j} = (x_{j+1},v_{j+1})$ in the above expression, and draw a sample. Note that this procedure can be efficiently parallelized.

\section{Remaining experimental details and results}\label{app:experimental-details}
\begin{table}[H]
\centering
\small
\caption{Average training errors for GoM, LV, and the NDO; lower is better.}
\label{tab:low-dimensional-training-errors}
\setlength{\tabcolsep}{7pt}
\begin{tabular}{lccc}
Method & GoM ($W_2$) & LV ($W_2$) & NDO ($W_1$) \\
\midrule
{OT-CFM} & $0.1319\stddev{0.0171}$ & $0.4778\stddev{0.0277}$ & $0.6748\mathbin{\pm}0.1759$ \\
{MMFM (prior)} & $0.104\stddev{0.015}$ & $0.400\stddev{0.052}$ & $2.0114\mathbin{\pm}1.3542$ \\
{MMFM} & $0.1301\stddev{0.0102}$ & $0.4968\stddev{0.0629}$ & $0.4933\mathbin{\pm}0.2087$ \\
AM (indep.) & $0.1063\stddev{0.0269}$ & $0.3920\stddev{0.1599}$ & $0.2032\mathbin{\pm}0.0906$ \\
\end{tabular}
\end{table}

\subsection{Gulf of Mexico}\label{app:gom}
The Gulf of Mexico (GoM) dataset describes the movement of water particles in ocean currents. In our experiments, we work with preprocessed data from \citet{shen2024multi}. They used data from a HYbrid Coordinate Ocean Model (HYCOM) reanalysis \citep{hycom2024} at a specific time snapshot to compute the velocity field of the ocean currents. Then they generated $999$ trajectories in $\R^2$ and randomly sampled each of them at one of the nine time points. With this procedure, each particle is observed only at a single time point and each generated marginal contains $n=111$ samples. The dataset is divided into $5$ training marginals (the even times $t_0, t_2, t_4, t_6, t_8$) with the remaining $4$ times $(t_1, t_3, t_5, t_7)$ used for held-out evaluation. The data are centered and scaled to have zero mean and unit variance.

In Tables~\ref{tab:gom-training-marginals} and~\ref{tab:gom-eval-marginals}, we present the $2$-Wasserstein distance on training and held-out marginals, to have a more precise indicator of the performance of the algorithms relative to Table~\ref{tab:low-dimensional-w2-300-official}. As for the other datasets, precise training information can be found in Table~\ref{tab:params-acceleration}. Figures~\ref{fig:traj_gom} and~\ref{fig:traj_lv} display 20 generated trajectories.

\begin{figure}[!htbp]
    \centering
    \includegraphics[width=\linewidth]{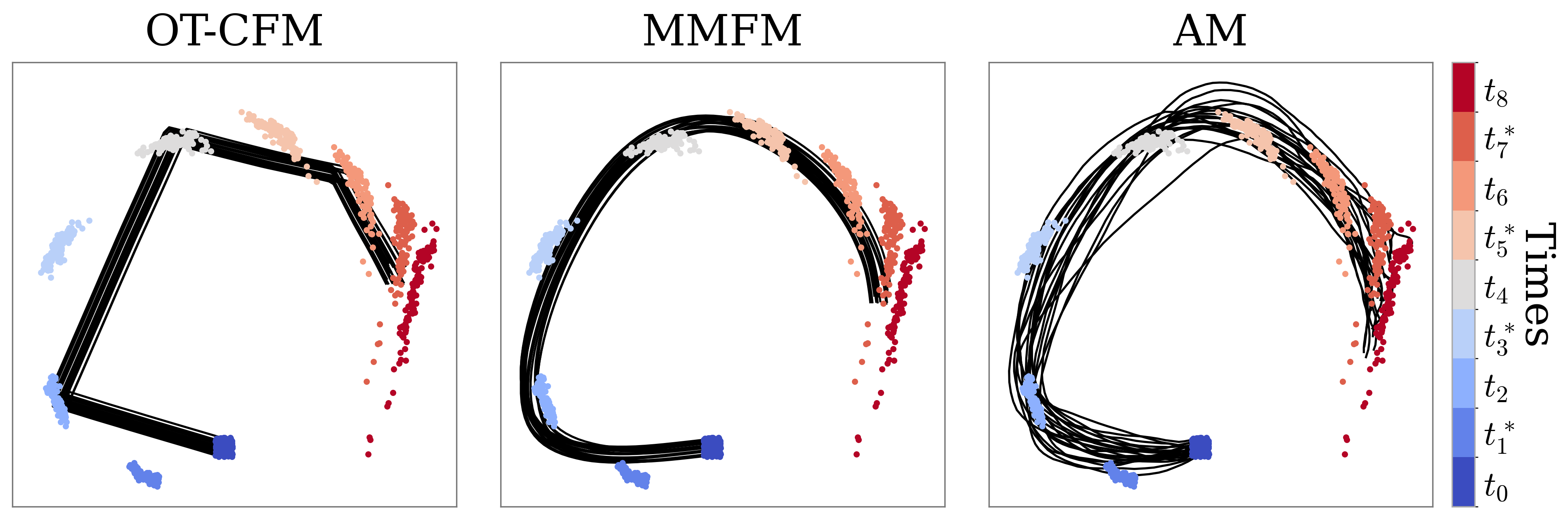}
    \caption{Generated trajectories for the GoM dataset. Stars mark held-out times. }
    \label{fig:traj_gom}
\end{figure}

\begin{table}[!htbp]
\centering
\caption{Per-timepoint $W_2$ at training marginals on the Gulf of Mexico dataset.}
\label{tab:gom-training-marginals}
\begin{tabular}{lcccc}
Method & $t_2$ & $t_4$ & $t_6$ & $t_8$ \\
\midrule
OT-CFM & $0.054\stddev{0.003}$ & $0.081\stddev{0.015}$ & $0.149\stddev{0.035}$ & $0.244\stddev{0.027}$ \\
MMFM (prior) & \wentry{0.050}{0.008} & \wentry{0.089}{0.017} & \wentry{0.140}{0.028} & \wentry{0.240}{0.042} \\
MMFM & $0.080\stddev{0.015}$ & $0.098\stddev{0.022}$ & $0.119\stddev{0.010}$ & $0.223\stddev{0.012}$ \\
AM & $0.068\stddev{0.014}$ & $0.094\stddev{0.038}$ & $0.061\stddev{0.026}$ & $0.202\stddev{0.072}$ \\
\end{tabular}
\end{table}

\begin{table}[!htbp]
\centering
\caption{Per-timepoint $W_2$ at held-out marginals on the Gulf of Mexico dataset.}
\label{tab:gom-eval-marginals}
\begin{tabular}{lcccc}
Method & $t_1$ & $t_3$ & $t_5$ & $t_7$ \\
\midrule
\rowcolor{quotedliterature}
3MSBM$^*$ & 0.200 & 0.180 & 0.070 & 0.090 \\
\midrule
OT-CFM & $0.268\stddev{0.003}$ & $0.337\stddev{0.025}$ & $0.203\stddev{0.031}$ & $0.229\stddev{0.043}$ \\
MMFM (prior) & \wentry{0.332}{0.010} & \wentry{0.187}{0.037} & \wentry{0.132}{0.030} & \wentry{0.187}{0.051} \\
MMFM & $0.204\stddev{0.022}$ & $0.193\stddev{0.033}$ & $0.115\stddev{0.020}$ & $0.169\stddev{0.015}$ \\
AM & $0.134\stddev{0.010}$ & $0.176\stddev{0.029}$ & $0.104\stddev{0.022}$ & $0.150\stddev{0.025}$ \\
\end{tabular}
\end{table}

\subsection{Lotka--Volterra}\label{app:lv}
This dataset, which simulates the Lotka--Volterra equations of a predator-prey model \citep{goel1971volterra}, is obtained from \citet{shen2024multi}. The data live in $\R^2$ and consist of $9$ marginals with $50$ points each. As with the GoM dataset, the even times are used for training ($t_0, t_2, t_4, t_6, t_8$) and the remaining $4$ for evaluation ($t_1, t_3, t_5, t_7$). Again, the data are centered and normalized. Tables~\ref{tab:lv-training-marginals} and~\ref{tab:lv-eval-marginals} respectively contain the training and evaluation errors in $W_2$ across the times of interest. 
\begin{table}[!htbp]
\centering
\caption{Per-timepoint $W_2$ at training marginals on the Lotka--Volterra dataset.}
\label{tab:lv-training-marginals}
\begin{tabular}{lcccc}
Method & $t_2$ & $t_4$ & $t_6$ & $t_8$ \\
\midrule
{OT-CFM} & {$0.196\stddev{0.035}$} & {$0.225\stddev{0.039}$} & {$0.423\stddev{0.024}$} & {$1.067\stddev{0.046}$} \\
{MMFM (prior)} & \wentry{0.179}{0.031} & \wentry{0.223}{0.039} & \wentry{0.448}{0.051} & \wentry{1.150}{0.216} \\
{MMFM} & {$0.180\stddev{0.030}$} & {$0.229\stddev{0.046}$} & {$0.441\stddev{0.053}$} & {$1.137\stddev{0.132}$} \\
{AM} & {$0.146\stddev{0.020}$} & {$0.220\stddev{0.087}$} & {$0.407\stddev{0.224}$} & {$0.795\stddev{0.422}$} \\
\end{tabular}
\end{table}

\begin{table}[!htbp]
\centering
\caption{Per-timepoint $W_2$ at held-out marginals on the Lotka--Volterra dataset. }
\label{tab:lv-eval-marginals}
\begin{tabular}{lcccc}
Method & $t_1$ & $t_3$ & $t_5$ & $t_7$ \\
\midrule
\rowcolor{quotedliterature}
3MSBM$^*$ & 0.240 & {0.170} & {0.150} & 0.360 \\
\midrule
{OT-CFM} & {$0.842\stddev{0.053}$} & {$0.323\stddev{0.092}$} & {$0.417\stddev{0.091}$} & {$0.940\stddev{0.056}$} \\
{MMFM (prior)} & \wentry{0.206}{0.039} & \wentry{0.211}{0.060} & \wentry{0.351}{0.017} & \wentry{0.941}{0.109} \\
{MMFM} & {$0.494\stddev{0.013}$} & {$0.288\stddev{0.049}$} & {$0.379\stddev{0.049}$} & {$0.924\stddev{0.070}$} \\
{AM} & {$0.506\stddev{0.029}$} & {$0.210\stddev{0.065}$} & {$0.257\stddev{0.087}$} & {$0.657\stddev{0.215}$} \\
\end{tabular}
\end{table}

\subsection{Nonlinear damped oscillator}\label{app:oscillator_app}
We now consider data generated by the second-order system
\begin{equation}\label{eq:oscillator-data}
\begin{gathered}
    \dd X_t = V_t\dd t\,,\quad
    \dd V_t = \bigl(-\nabla U(X_t)-V_t\bigr)\dd t+0.2\dd B_t\,,
\end{gathered}
\end{equation}
where the potential $U:\R^2\to\R$ is given by
\begin{align*}
U(x,y)=\tfrac12\bigl((2\pi)^2x^2+(3\pi)^2y^2\bigr)+2(x^4+y^4)+2x^2y^2\,.
\end{align*}
We suppose that initially $X_0\sim\gauss((1,0.6),0.1^2I)$ and $V_0\sim\gauss(0,0.05^2I)$, with $X_0\perp V_0$. To generate the data, we use the stochastic Heun integrator with stepsize $h=1/4096$ and simulate until $T = 10$ seconds. We generate 41 training marginals at times $t_j=j/4$, $j=0,\ldots,40$, and evaluate at the 20 held-out times $t_j^\star=0.125+0.5j$, $j=0,\ldots,19$, with 2,000 unpaired positions per snapshot. The data are centered and scaled using the training marginals.

\subsection{Embryoid body}\label{app:eb}
We use the preprocessed data from \cite{moon2019eb}, which are formed from the 100-dimensional PCA representation of the scRNA-seq of human embryoid cells collected across multiple time points. The dataset contains five marginals with 2,381, 4,163, 3,278, 3,665, and 3,332 cells (i.e., unpaired data) per marginal, respectively.

For our experiments, we used either the first $5$ principal components (resulting in EB5) or all 100 supplied PCs (resulting in EB100). In both cases, we used a coordinatewise standardization to center and scale the training marginals. For EB5, we perform a leave-one-out (LOO) procedure by training on four out of five marginals, following the protocol from \cite{tong2020trajectorynet}. The evaluation is then performed at the held-out time (one of $t_1$, $t_2$, $t_3$) using the $1$-Wasserstein distance on normalized data. In contrast, for our experiments on EB100, we train on $t_0,t_2,t_4$ and hold out marginals at times $t_1,t_3$ together, following \citet{theodoropoulos2025momentum}, and we report the errors after inverting the normalization. During training, we consider the full marginals and sample minibatches from them. At inference, we consider $2000$ points from the first marginal as starting points for the trajectories. Table~\ref{tab:eb5-holdout} reports the ${W}_1$ distances on the held-out marginals under the leave-one-out protocol.

\begin{figure}[!htbp]
    \centering
    \includegraphics[width=\linewidth]{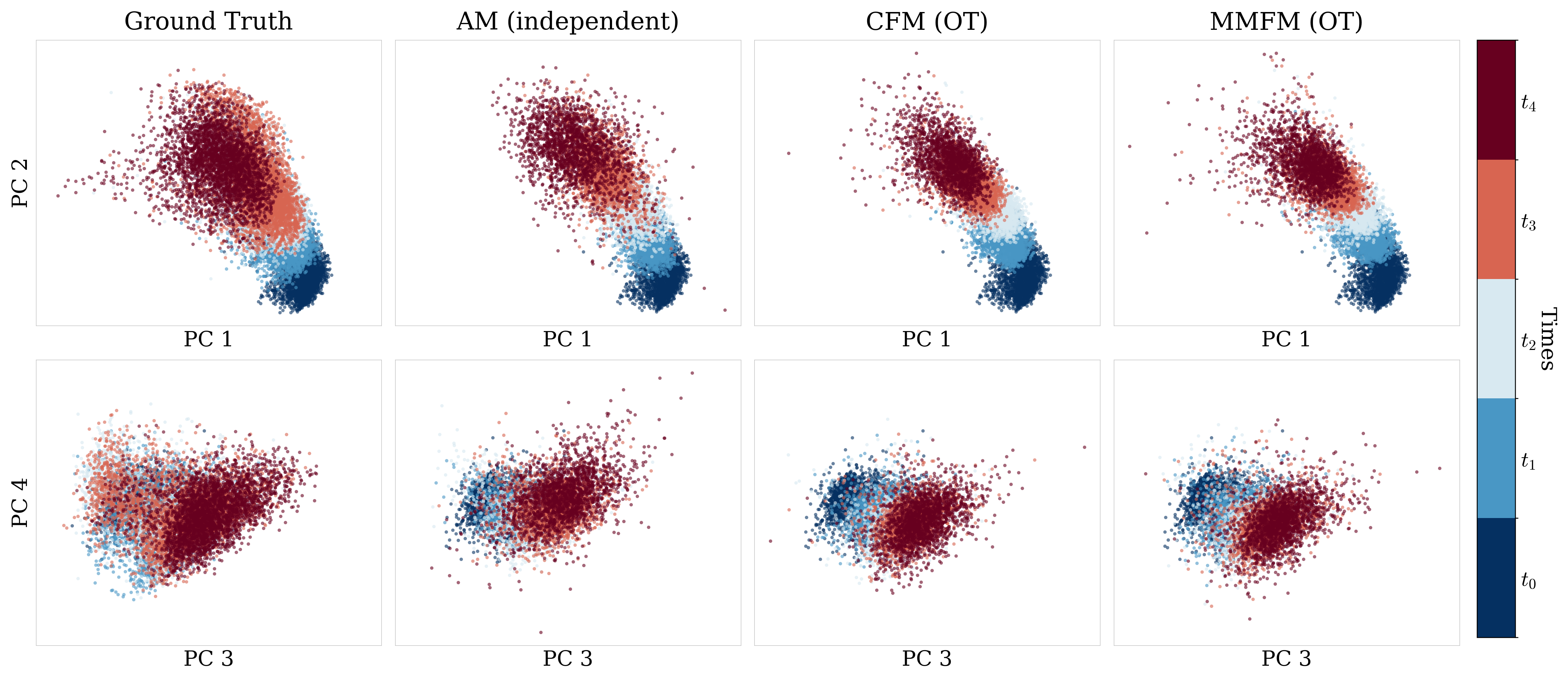}
    {
    \caption{{Observed and generated marginals for EB100.}}
    \label{fig:eb100-global-time-marginals}
    }
\end{figure}

\begin{table}[!htbp]
\centering
\caption{Mean and standard deviation of the $W_1$ error at held-out marginals on the EB-5 dataset.}
\label{tab:eb5-holdout}
\begin{tabular}{@{}lccccc@{}}
{Method} & {Coupling} & {$t_1$ mean} & {$t_2$ mean} & {$t_3$ mean} & {Average $\pm$ seed SD} \\
\midrule
AM & {Independent} & {$0.770$} & {$0.778$} & {$0.859$} & {$0.802\mathbin{\pm}0.038$} \\
\midrule
{CFM} & {OT} & {$0.737$} & {$0.888$} & {$0.894$} & {$0.839\mathbin{\pm}0.011$} \\
{MMFM} & {OT} & {$0.683$} & {$0.817$} & {$0.837$} & {$0.779\mathbin{\pm}0.025$} \\
AM & {OT} & {$0.689$} & {$0.692$} & {$0.853$} & {$0.745\mathbin{\pm}0.029$} \\
\end{tabular}
\end{table}

\subsection{CITE}\label{app:cite}
This dataset \citep{burkhardtcite} consists of single-cell CITE-seq measurements collected at four time points over ten days. We considered both the $5$-dimensional (CITE5) and $50$-dimensional (CITE50) PCA projections, preprocessed following prior work by \cite{wang2026joint}. The training data are normalized to zero mean and divided by the maximum per-feature standard deviation, following \cite{neklyudov2024computational}. We evaluate all methods (considering the normalized space for CITE5 and the original PCA space for CITE50) under the leave-one-out (LOO) protocol as presented in \cite{kapusniak2024metric}, reporting $W_1$ on held-out marginals in Tables~\ref{tab:cite5-holdout} and~\ref{tab:cite50-holdout}.\looseness-1

\begin{figure}[!htbp]
    \centering
    \includegraphics[width=0.95\linewidth]{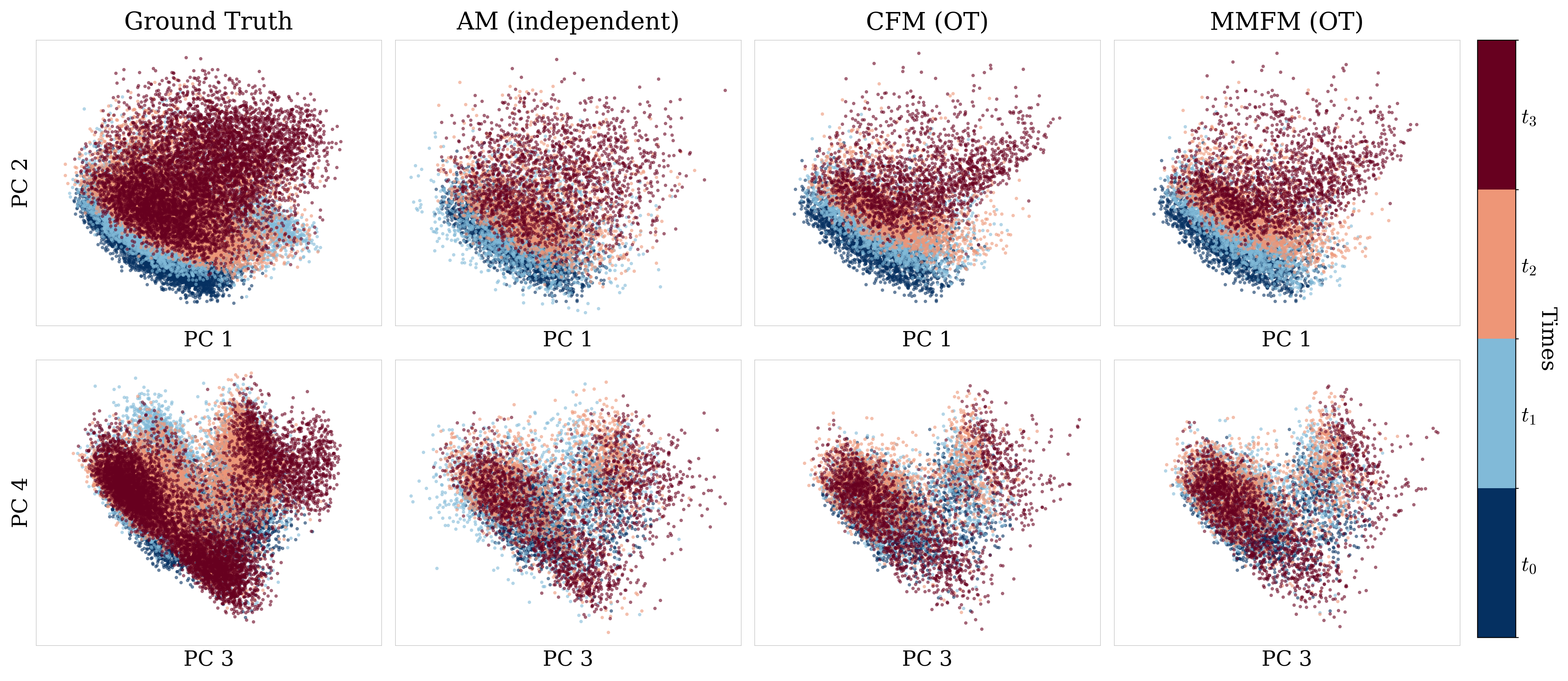}
    {
    \caption{{Observed and generated marginals for CITE5, with $t_1$ held out.}}
    \label{fig:cite5-global-time-marginals}
    }
\end{figure}

\begin{table}[!htbp]
\centering
\caption{Per-timepoint $W_1$ at held-out marginals on the CITE5 dataset (LOO protocol).}
\label{tab:cite5-holdout}
\begin{tabular}{@{}lcccc@{}}
{Method} & {Coupling} & {$t_1$ mean} & {$t_2$ mean} & {Average $\pm$ seed SD} \\
\midrule
AM & {Independent} & {$0.592$} & {$0.549$} & {$0.571\mathbin{\pm}0.026$} \\
\midrule
{CFM} & {OT} & {$0.520$} & {$0.463$} & {$0.491\mathbin{\pm}0.004$} \\
{MMFM} & {OT} & {$0.549$} & {$0.515$} & {$0.532\mathbin{\pm}0.008$} \\
AM & {OT} & {$0.549$} & {$0.547$} & {$0.548\mathbin{\pm}0.011$} \\
\end{tabular}
\end{table}

\begin{table}[!htbp]
\centering
\caption{Per-timepoint $W_1$ at held-out marginals on the CITE50 dataset (LOO protocol).}
\label{tab:cite50-holdout}
\begin{tabular}{@{}lcccc@{}}
{Method} & {Coupling} & {$t_1$ mean} & {$t_2$ mean} & {Average $\pm$ seed SD} \\
\midrule
AM & {Independent} & {$45.738$} & {$46.458$} & {$46.098\mathbin{\pm}0.996$} \\
\midrule
{CFM} & {OT} & {$39.563$} & {$38.711$} & {$39.137\mathbin{\pm}0.338$} \\
{MMFM} & {OT} & {$41.145$} & {$41.326$} & {$41.235\mathbin{\pm}0.417$} \\
AM & {OT} & {$42.048$} & {$44.086$} & {$43.067\mathbin{\pm}0.133$} \\
\end{tabular}
\end{table}

\section{{Training details and hyperparameters}}\label{app:training_details}
Table~\ref{tab:params-acceleration} describes the parameters used to learn the time-dependent acceleration or velocity fields. All models use the Adam optimizer~\citep{kingma2015adam} with $(\beta_1,\beta_2)=(0.9,0.999)$ for training. Errors are averaged over held-out times or folds within each seed, then summarized by the mean and population standard deviation across seeds. {The GoM and LV experiments} were run on an M2 MacBook Air; {NDO,} EB5, EB100, CITE5, and CITE50 experiments were run on a single NVIDIA H200 GPU.

\begin{table}[H]
\centering
\small
\setlength{\tabcolsep}{4pt}
\caption{Hyperparameters for training acceleration and velocities.}
\label{tab:params-acceleration}
{
\resizebox{\linewidth}{!}{%
{\begin{tabular}{lccccccc}
 & \textbf{GoM} & \textbf{LV} & \textbf{NDO} & \textbf{EB5} & \textbf{EB100} & \textbf{CITE5} & \textbf{CITE50} \\
\midrule
Architecture & MLP & MLP & MLP & MLP & MLP & MLP & MLP \\
Hidden width & 128 & 256 & 256 & 256 & 256 & 256 & 256 \\
Hidden layers & 4 & 2 & 3 & 5 & 5 & 5 & 4 \\
Batch size & 111 & 50 & 256 & 256 & 256 & 256 & 256 \\
Learning rate & $10^{-2}$ & $10^{-2}$ & $10^{-2}$ & $10^{-2}$ & $10^{-3}$ & $10^{-3}$ & $10^{-3}$ \\
Training iterations & 1000 & 500 & 5\,000 & {5\,000} & {5\,000} & {5\,000} & 5\,000 \\
Sampling steps & 100 & 100 & 20\,000 & 2\,000 & 2\,000 & 2\,000 & 2\,000 \\
\end{tabular}}%
}}
\end{table}

\subsection{{Initial-velocity samplers}}\label{app:sampler-ablation}
\subsubsection{Unconditional samplers}
As we described in Appendix~\ref{sec:cond_sampling}, when using the independent coupling, it is possible to decompose our initial velocity sampler into two steps. For a known sequence $(\alpha_0,\alpha_1,\ldots,\alpha_J)$ and data $X_0 \perp (X_{t_1},\ldots,X_{t_J})$, we want to sample from
\begin{align*}
    {\rm Law}(U + \zeta)\,,
\end{align*}
where $U = \sum_{j=1}^J \alpha_j X_{t_j}$ and $\zeta \sim \gauss(\alpha_0 X_0, \sigma_{\rm post}^2 I)$. Noticing that the latter step is trivial given that $\sigma_{\rm post}^2$ is also known in closed-form, our remaining step is to draw samples with the same law as $U$. To learn such a sampler, we generate training samples $\{\sum_{j=1}^J \alpha_j x_{t_j}^{(i)}\}_{i=1}^N$ from the training set and learn either a Gaussian mixture model with full component covariance matrices using standard Expectation-Maximization or a standard normalizing flow. (See the following section for more information on the latter, the only difference here is that the normalizing flow here is \emph{unconditional} which is a significantly easier learning problem). Table~\ref{tab:params-independent-sampler} has the relevant parameter choices across our datasets of interest.
\subsubsection{Conditional samplers}
{For pairwise OT coupling}, from the discussion in Section~\ref{sec:gen_trajec}, recall that we are given data of the form $\{(x_0^{(i)},v_0^{(i)})\}_{i=1}^n$ and we want to generate samples from $v$ given $x$. We denote an arbitrary conditional sampler as $q_\phi(\cdot\mid x_0)$.

Our first approach is based on conditional normalizing flows (NFs) with a standard Gaussian base and alternating coupling masks~{\citep{winkler2019learning}}. {Such a flow} is trained by minimizing the negative log-likelihood
\begin{equation}
\mathcal L_{\mathrm{NF}}(\phi)
\defeq -\mathbb E_{(X_0,V_0)}\!\left[\log q_\phi(V_0\mid X_0)\right]\,.
\label{eq:initial-velocity-nf-loss}
\end{equation}

Our second approach uses conditional FM~\citep{lipman2022flow}. Here, the sampler $q_\phi(\cdot\mid x_0)$ is obtained by numerically integrating (via midpoint discretization) a velocity field ${u_\phi}$, which takes in the concatenated input $(v_s,X_0,s)$. To train the sampler, we set $v_s=(1-s)Z+sV_0$ for $Z \sim \cN(0,I)$ and $s \sim {\rm Unif}(0,1)$, and minimize
\begin{equation}
\mathcal L_{\mathrm{FM}}({\phi})
=\frac{1}{d}\mathbb E\!\left[\left\|{u_\phi}(v_s,X_0,s)-(V_0-Z)\right\|^2\right]\,,
\label{eq:initial-velocity-fm-loss}
\end{equation}
where $(X_0,V_0)$ are drawn from the training data. 

Table~\ref{tab:params-conditional-sampler} gives the parameters for learning the relevant networks. We perform a brief ablation study, comparing the two samplers for OT-AM on four of the datasets both using $20\,000$ sampler updates in this ablation (the main EB5 and EB100 results use $3\,000$ NF updates); see Table~\ref{tab:sampler-ablation-ot}. We see that conditional NF performs better on EB5 and EB100; the low-dimensional results are mixed.

\begin{table}[!htbp]
\centering
\begingroup
\setlength{\tabcolsep}{4pt}
\caption{OT-AM initial-velocity sampler ablation: held-out errors ($\downarrow$).}
\label{tab:sampler-ablation-ot}
\begin{tabular}{@{}lcccc@{}}
Sampler & GoM ($W_2$) & LV ($W_2$) & EB5 ($W_1$) & EB100 ($W_1$) \\
\midrule
NF & $0.1301\mathbin{\pm}0.0194$ & $0.3838\mathbin{\pm}0.0549$ & $0.7408\mathbin{\pm}0.0151$ & $13.1696\mathbin{\pm}0.1187$ \\
FM & $0.1233\mathbin{\pm}0.0139$ & $0.4366\mathbin{\pm}0.0792$ & $0.7464\mathbin{\pm}0.0128$ & $13.3236\mathbin{\pm}0.0737$ \\
\end{tabular}
\endgroup
\end{table}

\begin{table}[!htbp]
\centering
\setlength{\tabcolsep}{3pt}
\caption{Hyperparameters of the initial-velocity samplers for AM.
}
\label{tab:params-independent-sampler}
\resizebox{0.8\linewidth}{!}{%
\begin{tabular}{lccccccc}
 & GoM & LV & NDO & EB5 & EB100 & CITE5 & CITE50 \\
\midrule
\multicolumn{8}{l}{Reference process} \\
Prior variance $\sigma_v^2$ & 1 & 50 & 0.005 & 0.005 & 0.005 & 0.01 & 1 \\
Noise amplitude $\sqrt{\varepsilon}$ & 4 & 2 & 0.2 & 0.2 & 0.2 & 0.01 & 0.01 \\
\midrule
\multicolumn{8}{l}{Unconditional Gaussian mixture (GMM)} \\
Num. components & 8 & 8 & --- & 8 & --- & 8 & --- \\
Samples used & 5\,000 & 5\,000 & --- & 20\,000 & --- & 20\,000 & --- \\
Maximum EM iterations & 300 & 300 & --- & 300 & --- & 300 & --- \\
\midrule
\multicolumn{8}{l}{Unconditional normalizing flow (NF)} \\
Coupling layers & --- & --- & 6 & --- & 6 & --- & 6 \\
Conditioner hidden width/layers & --- & --- & 256/2 & --- & 256/2 & --- & 256/2 \\
Activation & --- & --- & ReLU & --- & ReLU & --- & ReLU \\
Batch size & --- & --- & 256 & --- & 256 & --- & 256 \\
Learning rate & --- & --- & $10^{-3}$ & --- & $10^{-3}$ & --- & $10^{-3}$ \\
Training updates & --- & --- & 5\,000 & --- & 3\,000 & --- & 3\,000 \\
\end{tabular}%
}
\end{table}

\begin{table}[!htbp]
\centering
\setlength{\tabcolsep}{3pt}
\caption{Hyperparameters for training conditional velocity samplers for OT-AM.
}
\label{tab:params-conditional-sampler}
\resizebox{0.8\linewidth}{!}{{%
\begin{tabular}{lcccccc}
 & GoM & LV & EB5 & EB100 & CITE5 & CITE50 \\
\midrule
\multicolumn{7}{l}{Reference process} \\
Prior variance $\sigma_v^2$ & {1} & 50 & 0.005 & 0.005 & 0.01 & 1 \\
Noise amplitude $\sqrt{\varepsilon}$ & 4 & 2 & 0.2 & 0.2 & 0.01 & 0.01 \\
\midrule
\multicolumn{7}{l}{Conditional normalizing flow (NF)} \\
Coupling layers & 6 & 6 & 6 & 6 & 6 & 6 \\
Conditioner hidden width/layers & 256/2 & 256/2 & 256/2 & 256/2 & 256/2 & 256/2 \\
Activation & ReLU & ReLU & ReLU & ReLU & ReLU & ReLU \\
Batch size & 111 & 50 & 256 & 256 & 256 & 256 \\
Learning rate & $10^{-3}$ & $10^{-3}$ & $10^{-3}$ & $10^{-3}$ & $10^{-3}$ & $10^{-3}$ \\
Training updates & 20\,000 & 20\,000 & 3\,000 & 3\,000 & 3\,000 & 3\,000 \\
\midrule
\multicolumn{7}{l}{Conditional flow matching (FM)} \\
Hidden width & 256 & 256 & 256 & 256 & --- & --- \\
Hidden layers & 2 & 2 & 4 & 4 & --- & --- \\
Activation & SiLU & SiLU & SiLU & SiLU & --- & --- \\
Batch size & 256 & 256 & 256 & 256 & --- & --- \\
Learning rate & $10^{-3}$ & $10^{-3}$ & $10^{-3}$ & $10^{-3}$ & --- & --- \\
Training updates & 20\,000 & 20\,000 & 20\,000 & 20\,000 & --- & --- \\
ODE solver & Midpoint & Midpoint & Midpoint & Midpoint & --- & --- \\
Sampling steps & 32 & 32 & 32 & 32 & --- & --- \\
Network evaluations & 64 & 64 & 64 & 64 & --- & --- \\
\end{tabular}%
}}
\end{table}

\subsection{Training times}\label{app:training-times}
Table~\ref{tab:am-runtime-comparison} reports measured training times on a single NVIDIA GeForce GTX 1080 GPU. On the low dimensional datasets, we see that OT-CFM is slightly faster than AM while AM obtains lower held-out error with respect to OT-CFM and MMFM (see Table~\ref{tab:low-dimensional-w2-300-official}). The computational advantage of AM is more evident when moving to the single-cell datasets which are in medium to high dimension, in which our method with the independent coupling is the fastest. The AM version that relies on the OT coupling is a little slower, but has a training time comparable to OT-CFM. If we consider the errors reported in Table~\ref{tab:summary-high-dim} when interpreting these runtime results, we can see that the advantage of AM is to match closely the prescribed marginals, while reducing total training time by factors of up to approximately $5$ relative to OT-CFM and $27$ relative to MMFM (on EB5).

\begin{table}[tbp]
    \centering
    \caption[Training-time comparison]{Training-time comparison in seconds across datasets and methods.}
    \label{tab:am-runtime-comparison}
    \begingroup
    \footnotesize
    \setlength{\tabcolsep}{4pt}
    \renewcommand{\arraystretch}{1.15}
    \begin{tabular}{@{}lccccc@{}}
        \multirow{2}{*}{Dataset}
        & \multirow{2}{*}{Method}
        & \multicolumn{2}{c}{Auxiliary stages}
        & \multirow{2}{*}{Main network}
        & \multirow{2}{*}{Total} \\
        \cmidrule(lr){3-4}
        & & OT coupling & Initializer fitting & & \\
        \midrule
        \multirow{3}{*}{GoM}
        & AM (indep.) & --- & \(0.488 \pm 0.123\)
        & \(8.174 \pm 0.066\) & 8.662 \\
        & OT-CFM & 0.012 & ---
        & \(\mathbf{4.073} \pm 0.014\) & \textbf{4.085} \\
        & {MMFM} & 0.012 & ---
        & \(34.037 \pm 0.072\) & 34.049 \\
        \midrule
        \multirow{3}{*}{LV}
        & AM (indep.) & --- & \(0.827 \pm 0.326\)
        & \(2.630 \pm 0.053\) & 3.457 \\
        & OT-CFM & 0.005 & ---
        & \(\mathbf{1.438} \pm 0.016\) & \textbf{1.443} \\
        & {MMFM} & 0.004 & ---
        & \(8.374 \pm 0.054\) & 8.377 \\
        \midrule
        \multirow{4}{*}{CITE5}
        & AM (indep.) & --- & \(6.000 \pm 0.425\)
        & \(\mathbf{33.923} \pm 0.533\) & \textbf{39.924} \\
        & {AM (OT)} & 67.439 & \(50.911 \pm 1.346\)
        & \(79.573 \pm 2.979\) & 197.923 \\
        & OT-CFM & 67.873 & ---
        & \(118.339 \pm 0.722\) & 186.212 \\
        & {MMFM} & 67.873 & ---
        & \(759.351 \pm 4.594\) & 827.224 \\
        \midrule
        \multirow{4}{*}{CITE50}
        & AM (indep.) & --- & \(52.749 \pm 3.281\)
        & \(\mathbf{46.583} \pm 0.214\) & \textbf{99.333} \\
        & {AM (OT)} & 81.428 & \(57.535 \pm 2.273\)
        & \(99.560 \pm 0.452\) & 238.522 \\
        & OT-CFM & 83.296 & ---
        & \(142.667 \pm 14.347\) & 225.963 \\
        & {MMFM} & 83.296 & ---
        & \(882.507 \pm 43.834\) & 965.803 \\
        \midrule
        \multirow{4}{*}{EB5}
        & AM (indep.) & --- & \(4.896 \pm 0.516\)
        & \(\mathbf{37.790} \pm 0.823\) & \textbf{42.686} \\
        & {AM (OT)} & 15.044 & \(60.667 \pm 2.147\)
        & \(102.298 \pm 0.180\) & 178.008 \\
        & OT-CFM & 15.035 & ---
        & \(225.618 \pm 0.658\) & 240.652 \\
        & {MMFM} & 15.035 & ---
        & \(1155.470 \pm 2.289\) & 1170.505 \\
        \midrule
        \multirow{4}{*}{EB100}
        & AM (indep.) & --- & \(53.186 \pm 2.407\)
        & \(\mathbf{66.736} \pm 1.675\) & \textbf{119.923} \\
        & {AM (OT)} & 9.629 & \(56.608 \pm 0.923\)
        & \(114.793 \pm 0.768\) & 181.029 \\
        & OT-CFM & 9.529 & ---
        & \(128.503 \pm 1.308\) & 138.032 \\
        & {MMFM} & 9.529 & ---
        & \(822.615 \pm 2.358\) & 832.144 \\
    \end{tabular}
    \endgroup
\end{table}

\subsection{Sensitivity to the reference process parameters}\label{app:ablation-ref-param}
We now investigate the sensitivity of our trained models on some of the hyperparameters that govern the underlying dynamics while keeping the trained acceleration field fixed. Specifically, for the LV and GoM datasets, we vary noise amplitude $\sqrt{\epsilon}$ and prior variance $\sigma_v^2$ across orders of magnitude and investigate the impact on the learned trajectories and thus held-out error for learning the marginals.

Figures \ref{fig:heatmap-gom} and \ref{fig:heatmap-lv} visualize the parameter sweeps using a heat map of the model parameters; colors show held-out $W_2$ (lower is better). For GoM, the error varies more with noise amplitude than with reference variance, with the lowest plotted errors near $\sqrt{\varepsilon}=4$, whereas LV shows a much stronger interaction between the parameters. In particular, a wide range of parameters still give reasonable if not identical performance, which is shown in the trajectories below; see Figures~\ref{fig:gom-vel-var} and \ref{fig:gom-noise-var}.

\begin{figure}[!htbp]
    \centering
    \begin{subfigure}[t]{0.48\linewidth}
        \centering
        \includegraphics[width=\linewidth]{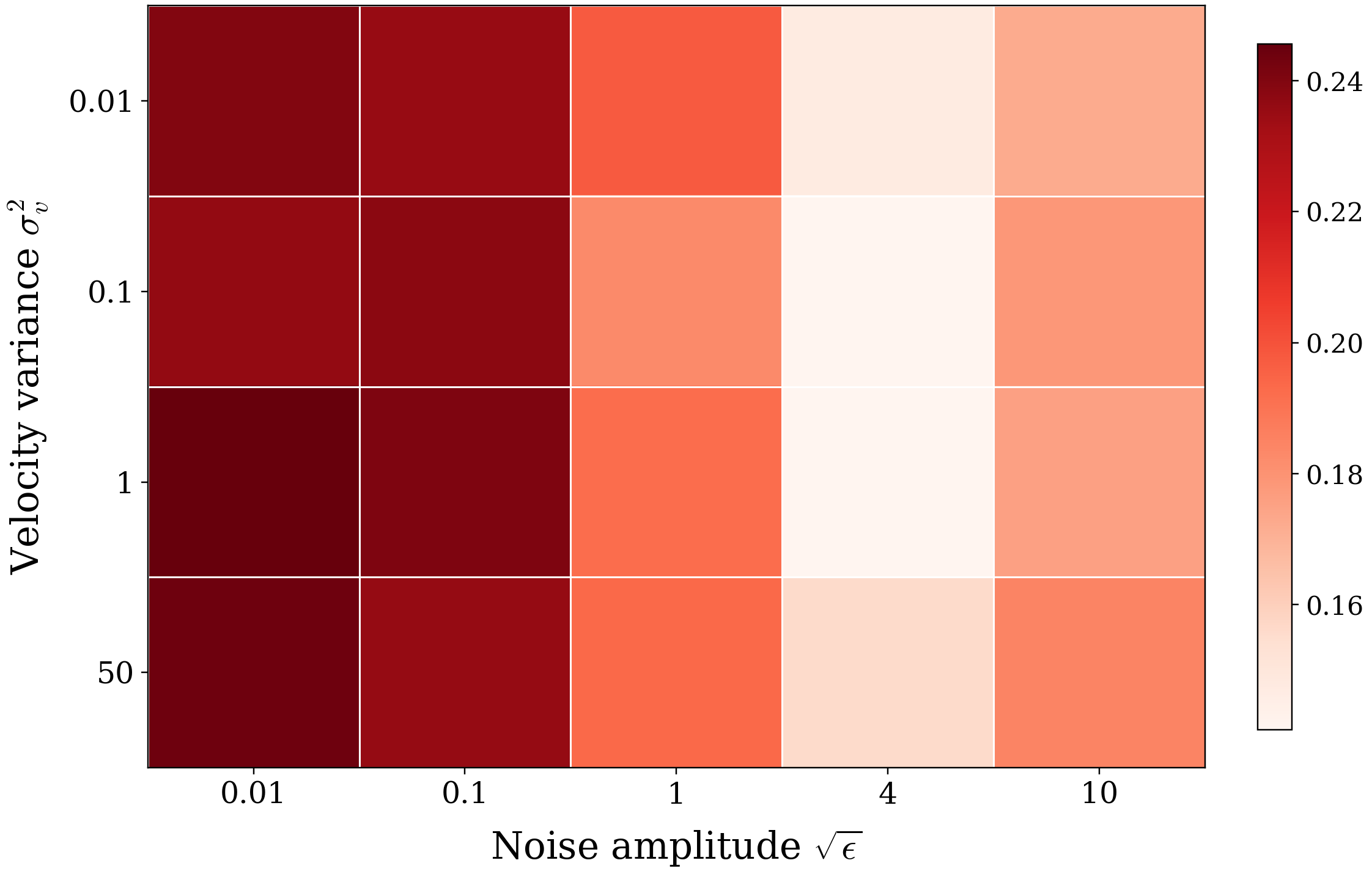}
        \caption{{GoM} dataset}
        \label{fig:heatmap-gom}
    \end{subfigure}\hfill
    \begin{subfigure}[t]{0.48\linewidth}
        \centering
        \includegraphics[width=\linewidth]{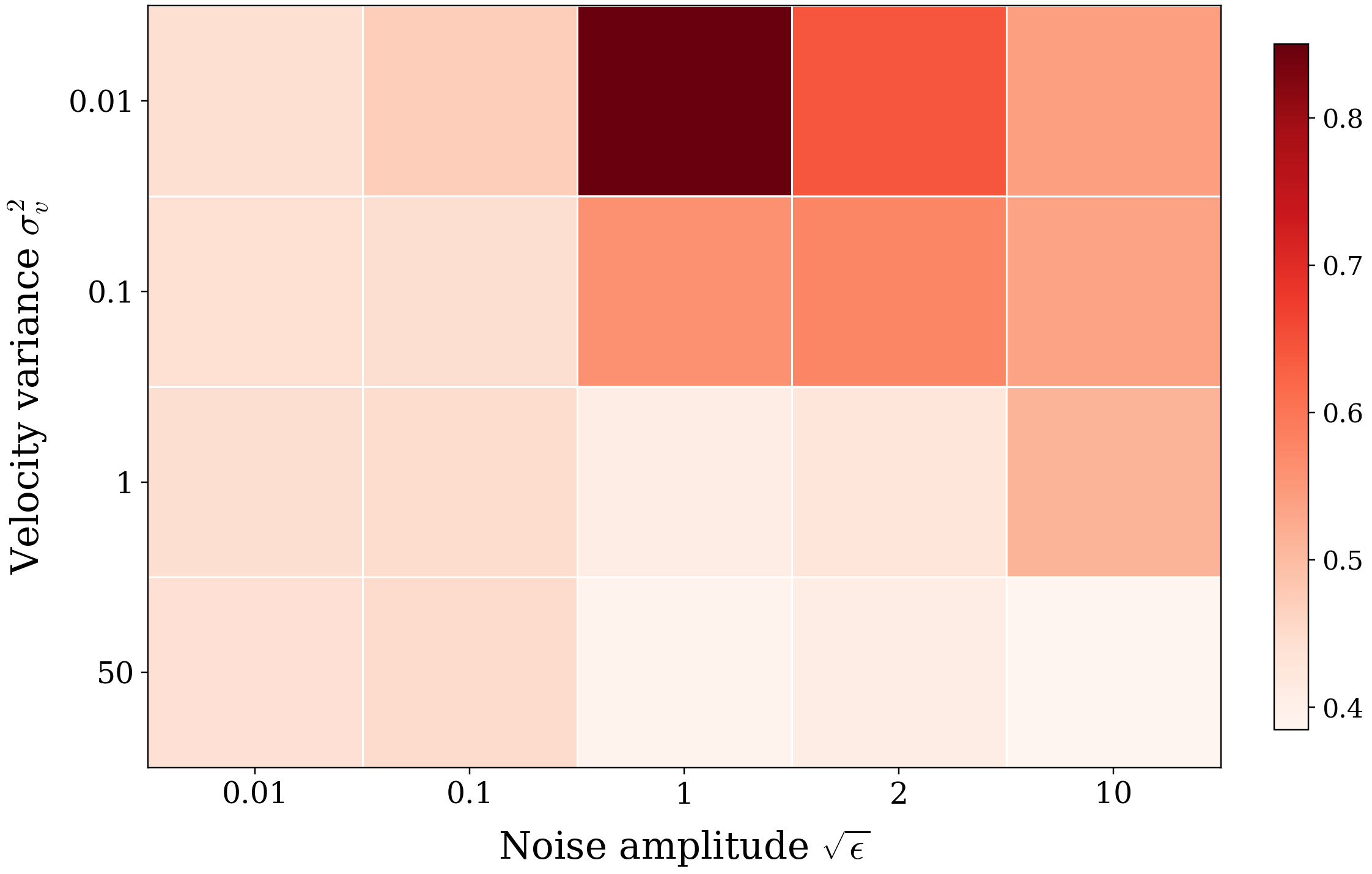}
        \caption{{LV} dataset}
        \label{fig:heatmap-lv}
    \end{subfigure}
    \caption{Sensitivity to hyperparameters in the reference process for GoM \textbf{(left)} and LV \textbf{(right)}; colors show held-out $W_2$ (where lower is better).}
    \label{fig:reference-parameter-heatmaps}
\end{figure}

\begin{figure}[!htbp]
    \centering
    \begin{subfigure}[t]{0.9\linewidth}
        \centering
        \includegraphics[width=\linewidth]{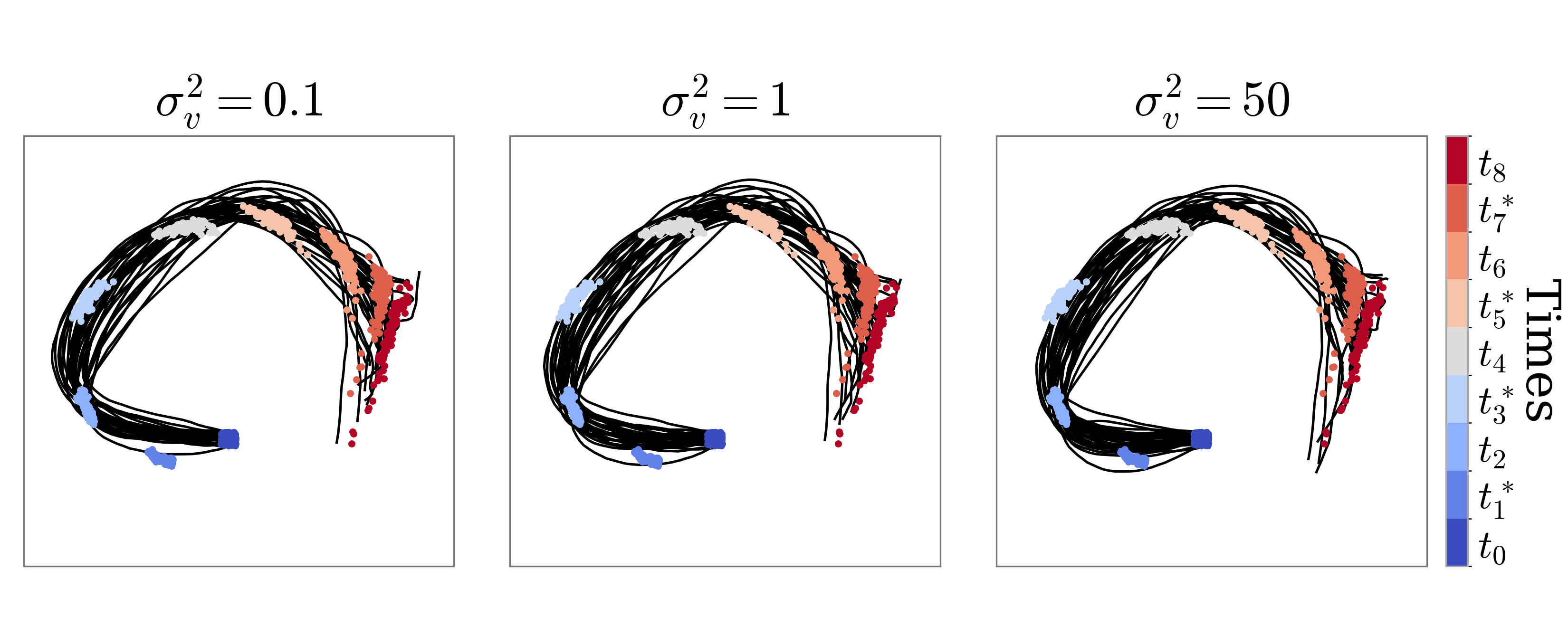}
        \caption{{Generated trajectories when varying prior variance $\sigma_v^2$ fixing $\sqrt{\varepsilon} = 4$.}}
        \label{fig:gom-vel-var}
    \end{subfigure}

    \medskip
    \begin{subfigure}[t]{0.9\linewidth}
        \centering
        \includegraphics[width=\linewidth]{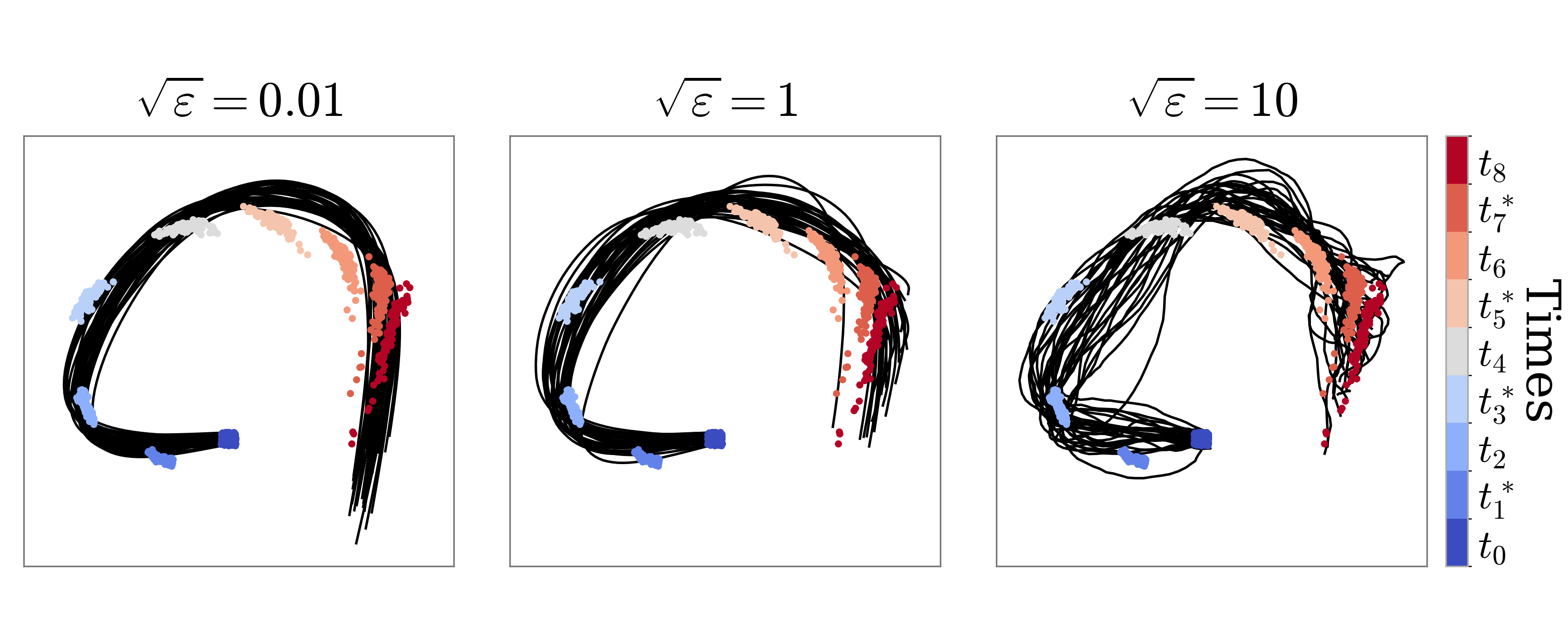}
        \caption{Generated trajectories when varying noise amplitude $\sqrt{\varepsilon}$ fixing $\sigma_v^2 = 1$.}
        \label{fig:gom-noise-var}
    \end{subfigure}
    \caption{{Generated GoM trajectories under hyperparameter changes in the reference process.}}
    \label{fig:gom-reference-parameter-trajectories}
\end{figure}

\end{document}